\documentclass[letterpaper, 10 pt, journal, twoside]{IEEEtran}

\IEEEoverridecommandlockouts                              

\usepackage{mathptmx} 
\usepackage{amsmath} 
\usepackage{amssymb}  
\usepackage{verbatim}
\usepackage{url}

\usepackage{cite}

\usepackage{graphicx}
\usepackage{multicol}

\usepackage[caption=false,font=footnotesize]{subfig}
\usepackage{enumerate}   
\usepackage{hyperref}
\hypersetup{
    colorlinks=true, 
    linktoc=all,     
    menucolor=black,
    linkcolor =black,
    anchorcolor =black,
    citecolor =black,
    urlcolor= black,
}

\usepackage{color}
\usepackage{pgfplots}
\pgfplotsset{compat=1.5}
\usepgfplotslibrary{statistics}
\pgfplotsset{grid style={dashed,gray}}
\usepackage[linesnumbered,ruled]{algorithm2e}
\usepackage[ampersand]{easylist}

\usepackage[utf8]{inputenc}

\def\Red#1{\textcolor{red}{#1}}

\usepackage{booktabs}

\newcommand{\vmax}{V_{\mbox{\scriptsize max}}}
\newcommand{\amax}{A_{\mbox{\scriptsize max}}}

\newcommand{\tcritic}{t^*}

\newcommand{\lx}{\lambda_{x,1}}
\newcommand{\lv}{\lambda_{v,1}}
\newcommand{\lxx}{\lambda_{x,2}}
\newcommand{\lvv}{\lambda_{v,2}}

\newcommand{\lxxx}{\lambda_{x,3}}

\newcommand{\trans}{T}

\newcommand{\pstart}{p_\mathrm{start}}
\newcommand{\pgoal}{p_\mathrm{goal}}

\newcommand{\vstart}{v_\mathrm{start}}
\newcommand{\vgoal}{v_\mathrm{goal}}

\newcommand{\astart}{a_\mathrm{start}}
\newcommand{\agoal}{a_\mathrm{goal}}

\newcommand{\xstart}{x_\mathrm{start}}
\newcommand{\xgoal}{x_\mathrm{goal}}

\newcommand{\lmin}{\ell_\mathrm{m}}

\newtheorem{theorem}{Theorem}

\newtheorem{lemma}[theorem]{Lemma}

\newtheorem{problem}[theorem]{Problem}

\title{
A Hybrid Method for Online Trajectory Planning of Mobile Robots in Cluttered Environments
}

 \author{Leobardo~Campos-Mac\'ias,~\IEEEmembership{Student Member,~IEEE,}
  David~Gómez-Gutiérrez,~\IEEEmembership{Member,~IEEE,} \\Rodrigo Aldana-López,~\IEEEmembership{Student Member,~IEEE,} Rafael~de~la~Guardia and José~I.~Parra-Vilchis

 \thanks{The authors are with the Multi-Agent Autonomous Systems Lab, Intel Labs, Intel Tecnología de México. (e-mail: \href{mailto:L.E.CamposMacias@ieee.org}{l.e.camposmacias@ieee.org}, \href{mailto:David.Gomez.G@ieee.org}{david.gomez.g@ieee.org}, \href{mailto:A.Rodrigo.AldanaLopez@ieee.org}{a.rodrigo.aldanalopez@ieee.org}, \href{mailto:rafael.de.la.guardia@intel.com}{rafael.de.la.guardia@intel.com}, \href{mailto:jose.i.parra.vilchis@intel.com}{jose.i.parra.vilchis@intel.com}).
  }

\thanks{A video of the experiments can be found at \url{https://youtu.be/DJ1IZRL5t1Q}}
\thanks{Digital Object Identifier (DOI):\href{http://dx.doi.org/10.1109/LRA.2017.2655145}{10.1109/LRA.2017.2655145}}

\thanks{\Red{
This is the author's accepted version of the manuscript: 
Campos-Macías, L., Gómez-Gutiérrez, D., Aldana-López, R., de la Guardia, R., \& Parra-Vilchis, J. I. (2017). A hybrid method for online trajectory planning of mobile robots in cluttered environments. IEEE Robotics and Automation Letters, 2(2), 935-942.
Please cite the publisher's version. For the publisher's version and full citation details see:
\url{https://doi.org/10.1109/LRA.2017.2655145}
}}
}

\begin{document}

\maketitle

\begin{abstract}
This paper presents a method for online trajectory planning in known environments. 
The proposed algorithm is a fusion of sampling-based techniques and model-based optimization via quadratic programming. The former is used to efficiently generate an obstacle-free path while the latter takes into account the robot dynamical constraints to generate a time-dependent trajectory. The main contribution of this work lies on the formulation of a convex optimization problem over the generated obstacle-free path that is guaranteed to be feasible. Thus, in contrast with previously proposed methods, iterative formulations are not required. The proposed method has been compared with state-of-the-art approaches showing a significant improvement in success rate and computation time. To illustrate the effectiveness of this approach for online planning, the proposed method was applied to the fluid autonomous navigation of a quadcopter in multiple environments consisting of up to two hundred obstacles. The scenarios hereinafter presented are some of the most densely cluttered experiments for online planning and navigation reported to date. See video at \url{https://youtu.be/DJ1IZRL5t1Q}.
\end{abstract}

\begin{IEEEkeywords}
Aerial Robotics, Autonomous Agents, Autonomous Vehicle Navigation, Collision Avoidance, Motion and Path Planning
\end{IEEEkeywords}
\section{Introduction}
\IEEEPARstart{T}{he} development of algorithms to enable mobile robotic systems to navigate in complex dynamic environments are of paramount importance towards fully autonomous systems performing complicated tasks. Recently, there has been great progress in online motion planning in structured environments, most notably for autonomous cars~\cite{Katrakazas2015,Paden2016}. However, to achieve similar levels of autonomy in three dimensional unstructured environments further development is required. Such capabilities are critical, for instance, when involved in search and rescue missions inside collapsed buildings.

Among the most efficient methods available for motion planning of mobile robots are sampling-based methods~\cite{Karaman2011,Elbanhawi2014,Gammell2015,Choudhury2016} and optimization-based methods~\cite{Deits2015,Landry2016}. In the former category, algorithms such as Rapidly-exploring Random Trees (\textit{RRT})~\cite{Karaman2011,Elbanhawi2014} or Batch Informed Trees (\textit{BIT*})~\cite{Gammell2015,Choudhury2016}, have demonstrated to be effective in determining obstacle-free paths online. However, their extension to trajectory planning taking into account the robot dynamical constraints often results in time consuming algorithms that are no longer applicable online~\cite{Webb2013,Xie2015}. In the latter category, optimization-based methods take advantage of the underlying mathematical model and have proven to be effective in handling the robot dynamical constraints~\cite{Richter2016,Augugliaro2012,Deits2015}. The downside is that optimization problems often result in nonlinear programs~\cite{Deits2015,Landry2016} or Sequential Convex Programs (\textit{SCP})~\cite{Augugliaro2012, Chen2015decoupled,Schulman2014} for which no efficient solvers exist. 
This characteristic makes them unsuitable for online planning in cluttered environments.

To take advantage of the efficiency of sampling-based algorithms and the inclusion of the dynamical constraints in optimization-based methods, hybrid approaches have been proposed~\cite{Chen2016,Richter2016,Zucker2013,Kalakrishnan2011}. Among the most promising are \textit{CHOMP}~\cite{Zucker2013} and \textit{STOMP}~\cite{Kalakrishnan2011}. Both methods rely on the optimization of an objective function with smoothness and collision costs solved by functional gradient descent in \textit{CHOMP} and by gradient-free candidate sampling in \textit{STOMP}. Unfortunately, in cluttered environments reported results show low success rate, often converging to local minima or altogether failing to find feasible solutions. While hybrid approaches are promising, the main challenge for their broader use in online planning is creating formulations and heuristics that translate into optimization problems that can be solved efficiently, avoiding incremental or iterative solutions.

\begin{figure}[t]
\begin{center}         
\includegraphics[width=8cm]{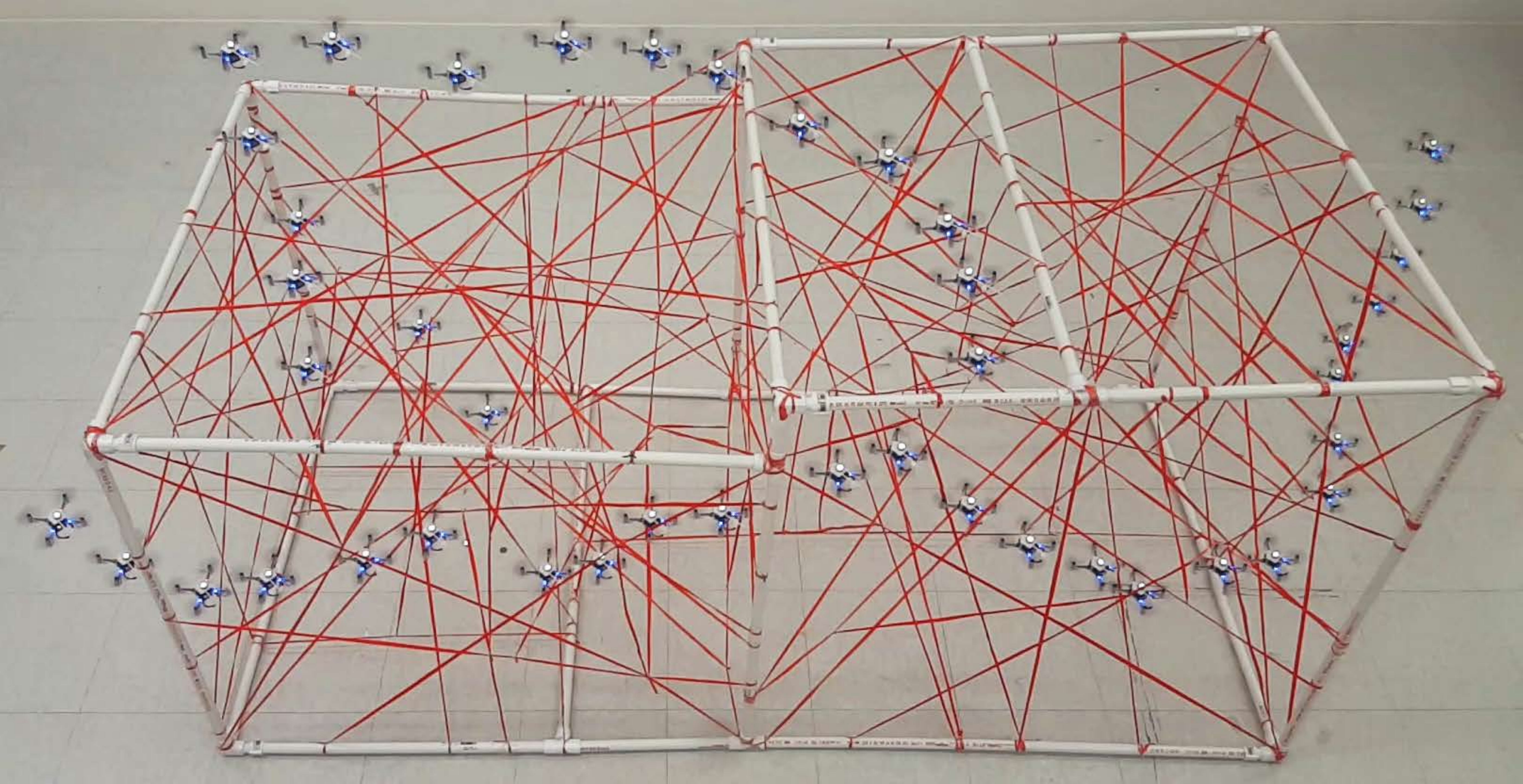}
\end{center}
\caption{Composite image of a quadcopter executing a trajectory planned in a fraction of a second for an obstacle-dense environment. See video at \url{https://youtu.be/DJ1IZRL5t1Q}.}
\label{Fig:ExperimentStrings}
\end{figure}

Regarding methods for trajectory planning in cluttered environments, in~\cite{Landry2016} semidefinite programming is used to segment the space into convex regions while a mixed-integer convex program determines the obstacle-free trajectory. Unfortunately results show infeasibility for online implementation since reported solutions are in the order of minutes. In~\cite{Chen2016} a free-space flight corridor is generated using an octree-based structure followed by a Quadratic Program (\textit{QP}) formulation using connected polynomial functions for each overlapping convex region. The method assumes that each segment end-time is given, though, no methodology for automatically generating such times is provided. This is a major drawback as a bad selection may result in infeasible problems which will require iterative \textit{QP} formulations, failing to provide an online solution. A similar approach is presented in \cite{Richter2016}, in which collision-free paths are obtained using the asymptotically optimal version of the \textit{RRT} algorithm (\textit{RRT*}) proposed in~\cite{Karaman2011}, followed by smooth trajectories generation using a \textit{QP} that has to be solved iteratively to determine the total time of the trajectory. Since iterative solutions are undesirable for online implementations and furthermore, as noted in~\cite{Chen2016}, collisions may still occur as the trajectory deviates from the original path. To overcome this problem, the authors proposed to add intermediate waypoints, but no method to determine the number of waypoints needed is provided. In~\cite{Schulman2014} a trajectory optimization algorithm inspired in \textit{CHOMP} was presented, which organizes the workspace in convex regions to perform a \textit{SCP}, improving the computation time of~\textit{CHOMP} but with the drawback that convex regions are hard to compute online. Also, inspired in~\textit{CHOMP}, in  \cite{Oleynikova2016} an online trajectory optimization method for local replanning was presented. Unfortunately, it often converges to local minima or fails to find feasible solutions, obtaining a low success rate compared to sampling-based methods.

In this paper, a hybrid method for online trajectory planning in cluttered environments is proposed. Based on a path generated by a sampling-based planner a \textit{QP} is defined taking into account the robot dynamical constraints. The proposed formulation ensures feasibility of the \textit{QP}, avoiding the need to solve multiple optimization problems in an iterative or sequential fashion~\cite{Augugliaro2012,Richter2016,Chen2015decoupled,Schulman2014}. 
Comparisons with state-of-the-art approaches, such as~\cite{Chen2015decoupled} and~\cite{Schulman2014}, are presented showing superior performance in computation time and success rate. To illustrate its effectiveness, the proposed method was applied for the online planning of a palm-sized quadcopter in multiple scenarios consisting of up to two hundred static obstacles. These scenarios are some of the most densely cluttered for online planning of a quadcopter reported in the literature to date. One of these experiments is illustrated in Fig.~\ref{Fig:ExperimentStrings}. 

\section{Outline of the Proposed Method}\label{sec:outline}

\subsection{Notation}

Let $\mathcal{B}(p)$ be the smallest ball, centered in the robot's centroid with position $p$, containing the mobile robot. The center of $\mathcal{B}(p)$ is modeled as a free particle in a non-rotating frame with state $x=\left[ p^T \ v^T \ a^T\right]^T$, where $[\cdot]^T$ is the transpose operator, $p$ is the position, $v$ is the velocity and  $a$ is the acceleration of the particle.
The configuration space of the particle is denoted by $\chi=[0,1]^d$, with $d\in\{2,3\}$ as its dimension. The obstacle region is $\chi_{\mathrm{obsReal}}$ and the obstacle-free space is defined as $\chi_{\mathrm{Free}}=\{p\in\chi|\mathcal{B}(p)\cap \chi_{\mathrm{obsReal}}=\emptyset \}$.

The notation $f(\cdot)$ is used to represent a continuous signal while $f[\cdot]$ represents a discrete signal. 
A path $\eta : \{0,\cdots,S\} \mapsto \chi_\mathrm{free}$ is a sequence of position nodes. Given a discrete path $\eta[s]$ a continuous path $\eta(\mathbf{s})$ can be obtained as the concatenation of the line segments connecting consecutive nodes. A trajectory $p(t)$ is a time-dependent sequence of positions, with velocity $v(t)$, acceleration $a(t)$ and jerk $j(t)$. Continuous-time is represented by $t$ while $k$ represents discrete-time. Thus, for a time step $h$, $p(t)$ is a continuous-time trajectory while $p[k]$ is a discrete-time trajectory such that $p(kh)=p[k]$. 

The set of dynamical constraints defines a convex set of allowed states $x=\left[ p^T \ v^T \ a^T\right]^T$ that is denoted by $\mathcal{X}_{\mathrm{allowed}}$.

\subsection{Problem statement and outline of the method}

\begin{figure}
\begin{center}
	\def\svgwidth{8.5cm} 
\begingroup%
  \makeatletter%
  \providecommand\color[2][]{%
    \errmessage{(Inkscape) Color is used for the text in Inkscape, but the package 'color.sty' is not loaded}%
    \renewcommand\color[2][]{}%
  }%
  \providecommand\transparent[1]{%
    \errmessage{(Inkscape) Transparency is used (non-zero) for the text in Inkscape, but the package 'transparent.sty' is not loaded}%
    \renewcommand\transparent[1]{}%
  }%
  \providecommand\rotatebox[2]{#2}%
  \ifx\svgwidth\undefined%
    \setlength{\unitlength}{370.24993515bp}%
    \ifx\svgscale\undefined%
      \relax%
    \else%
      \setlength{\unitlength}{\unitlength * \real{\svgscale}}%
    \fi%
  \else%
    \setlength{\unitlength}{\svgwidth}%
  \fi%
  \global\let\svgwidth\undefined%
  \global\let\svgscale\undefined%
  \makeatother%
  \begin{picture}(1,0.95)%
  \footnotesize{
	\put(0.415,0.9){\color[rgb]{0,0,0}\makebox(0,0)[lb]{\smash{$\xstart\ $ $\xgoal$}}}%
    \put(0.36,0.805){\color[rgb]{0,0,0}\makebox(0,0)[lb]{\smash{Generation of $\chi_\mathrm{obs}$}}}%
    \put(0.4,0.715){\color[rgb]{0,0,0}\makebox(0,0)[lb]{\smash{Path Planning}}}%
    \put(0.4,0.63){\color[rgb]{0,0,0}\makebox(0,0)[lb]{\smash{Path Refining}}}%
    \put(0.34,0.54){\color[rgb]{0,0,0}\makebox(0,0)[lb]{\smash{Time step comp. and}}}%
    \put(0.34,0.5){\color[rgb]{0,0,0}\makebox(0,0)[lb]{\smash{Waypoints generation}}}%
    \put(0.35,0.42){\color[rgb]{0,0,0}\makebox(0,0)[lb]{\smash{Constraints and Obj.  }}}%
    \put(0.37,0.38){\color[rgb]{0,0,0}\makebox(0,0)[lb]{\smash{Func. formulation}}}%
    \put(0.42,0.3){\color[rgb]{0,0,0}\makebox(0,0)[lb]{\smash{Solve  \textit{QP}}}}%
    \put(0.36,0.215){\color[rgb]{0,0,0}\makebox(0,0)[lb]{\smash{Feasible Trajectory}}}%
    \put(0.40,0.175){\color[rgb]{0,0,0}\makebox(0,0)[lb]{\smash{$p(t)$, $v(t)$, $a(t)$}}}%
    \put(0.47,0.09){\color[rgb]{0,0,0}\makebox(0,0)[lb]{\smash{is}}}%
    \put(0.4,0.05){\color[rgb]{0,0,0}\makebox(0,0)[lb]{\smash{$x(\bar{t})=\xgoal$?}}}%
    \put(0.28,0.08){\color[rgb]{0,0,0}\makebox(0,0)[lb]{\smash{YES}}}%
    \put(0.62,0.08){\color[rgb]{0,0,0}\makebox(0,0)[lb]{\smash{NO}}}%
    \put(0.15,0.06){\color[rgb]{0,0,0}\makebox(0,0)[lb]{\smash{END}}}%
    \put(0.08,0.49){\color[rgb]{0,0,0}\makebox(0,0)[lb]{\smash{Dynamical}}}%
    \put(0.08,0.44){\color[rgb]{0,0,0}\makebox(0,0)[lb]{\smash{constraints}}}%
    \put(0.795,0.47){\color[rgb]{0,0,0}\makebox(0,0)[lb]{\smash{Cond.}}}%
    \put(0.755,0.44){\color[rgb]{0,0,0}\makebox(0,0)[lb]{\smash{to replan in}}}%
    \put(0.765,0.4){\color[rgb]{0,0,0}\makebox(0,0)[lb]{\smash{$[\bar{t}+t_s,t_f]$?}}}%
    \put(0.78,0.55){\color[rgb]{0,0,0}\makebox(0,0)[lb]{\smash{NO}}}%
    \put(0.93,0.48){\color[rgb]{0,0,0}\makebox(0,0)[lb]{\smash{YES}}}%
    \put(0.7,0.83){\color[rgb]{0,0,0}\makebox(0,0)[lb]{\smash{$\xstart\leftarrow x(\bar{t}+t_s)$}}}%
	\put(0.7,0.79){\color[rgb]{0,0,0}\makebox(0,0)[lb]{\smash{$\xgoal\leftarrow \hat{x}_\textrm{goal}$}}}%
	\put(0.5,0.47){\color[rgb]{0,0,0}\makebox(0,0)[lb]{\smash{$t_f$}}}%
    \put(0.74,0.08){\color[rgb]{0,0,0}\makebox(0,0)[lb]{\smash{$\bar{t}\leftarrow t$}}}%
    \put(0.70,0.22){\color[rgb]{0,0,0}\makebox(0,0)[lb]{\smash{Commit. Trajectory}}}%
    \put(0.75,0.18){\color[rgb]{0,0,0}\makebox(0,0)[lb]{\smash{for $[\bar{t},\bar{t}+t_s]$ }}}%
  	\put(0.04,0.56){\color[rgb]{0,0,0}\makebox(0,0)[lb]{\smash{\textit{QP} Formulation}}}%
  	\put(0.04,0.74){\color[rgb]{0,0,0}\makebox(0,0)[lb]{\smash{Sampling-based}}}%
  	\put(0.04,0.71){\color[rgb]{0,0,0}\makebox(0,0)[lb]{\smash{Planning}}}%
  }
    \put(0.02,0.0){\includegraphics[width=\unitlength,page=1]{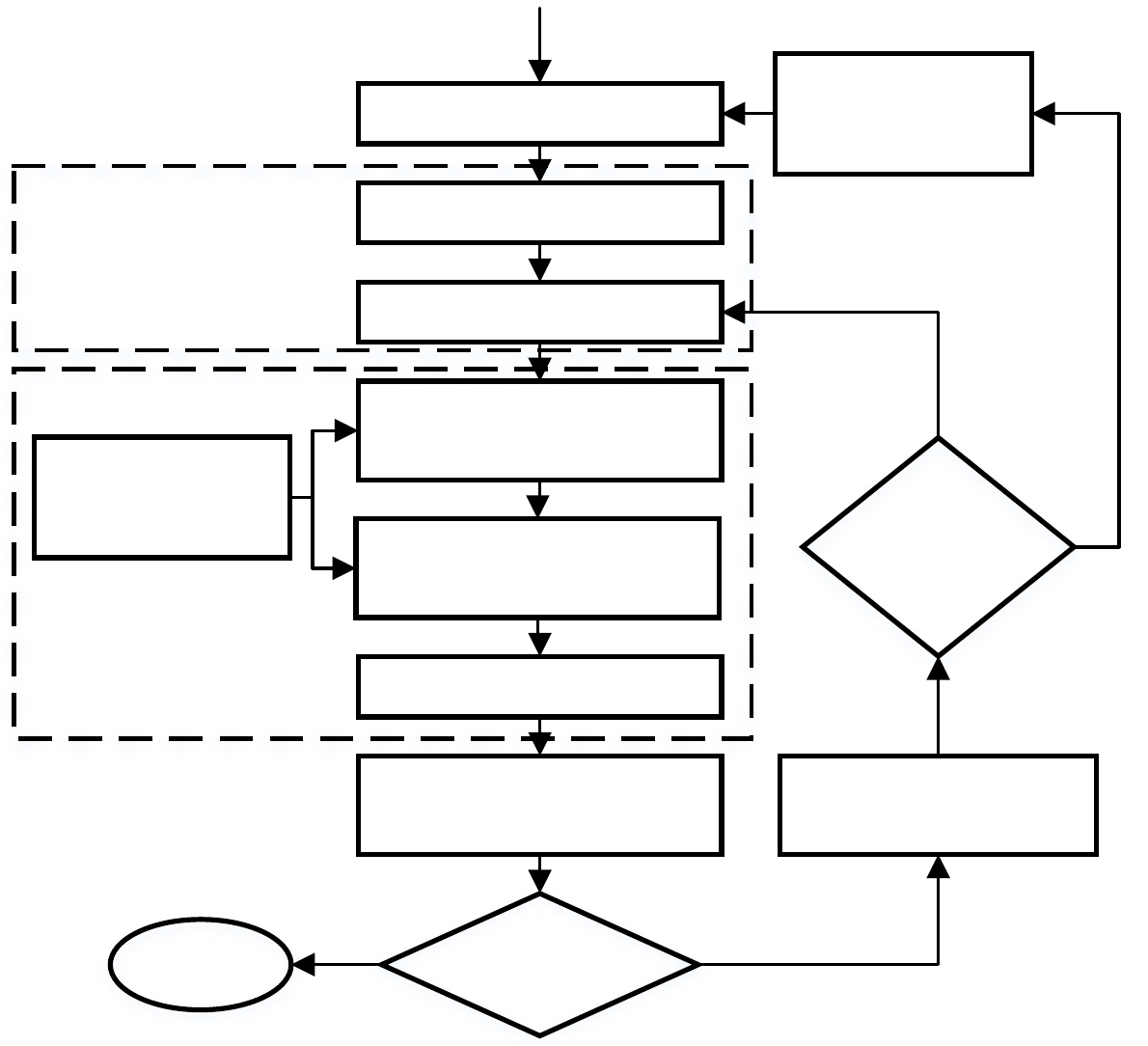}}%
  \end{picture}%
\endgroup%
\end{center}
\caption{Sketch of the proposed hybrid method for trajectory generation.}
\label{Fig:Approach}
\end{figure}
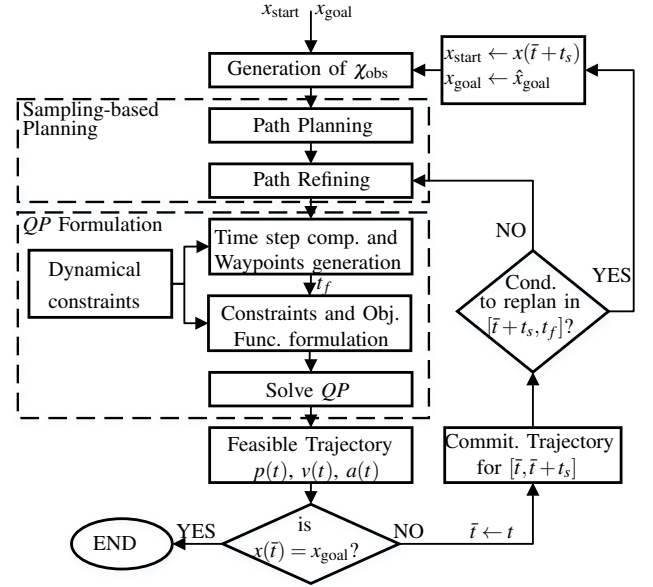

\begin{problem}
Consider a robot whose centroid is modeled as a free particle with dynamics described by $\dot{p}(t)=v(t)$, $\dot{v}(t)=a(t)$ where $p,v,a\in\mathbb{R}^d$ and $\|a(t)\|_\infty\leq\amax$. Assuming knowledge of the environment (i.e. the obstacle region $\chi_{\mathrm{obsReal}}$ is known) and given the initial and final states, 
 $\xstart=\left[ \pstart^T \ \vstart^T \ \astart^T\right]^T\in\mathcal{X}_{\mathrm{allowed}}$ and $\xgoal=\left[ \pgoal^T \ \vgoal^T \ \agoal^T\right]^T\in\mathcal{X}_{\mathrm{allowed}}$, respectively, find a trajectory and a time $t_f$ such that, 
$x(0)=\xstart$, $x(t_f)=\xgoal$ and for all time $t\in[0,t_f]$, $p(t)\in\chi_{\mathrm{Free}}$ and $x(t)\in\mathcal{X}_{\mathrm{allowed}}$. 
\label{ProbPlanning}
\end{problem}

\begin{figure*}
    \centering
    \subfloat[Generating a high-clearance obstacle-free path]{\label{subfig:Path}
    \def\svgwidth{5.95cm}
\begingroup%
  \makeatletter%
  \providecommand\color[2][]{%
    \errmessage{(Inkscape) Color is used for the text in Inkscape, but the package 'color.sty' is not loaded}%
    \renewcommand\color[2][]{}%
  }%
  \providecommand\transparent[1]{%
    \errmessage{(Inkscape) Transparency is used (non-zero) for the text in Inkscape, but the package 'transparent.sty' is not loaded}%
    \renewcommand\transparent[1]{}%
  }%
  \providecommand\rotatebox[2]{#2}%
  \ifx\svgwidth\undefined%
    \setlength{\unitlength}{170.07874016bp}%
    \ifx\svgscale\undefined%
      \relax%
    \else%
      \setlength{\unitlength}{\unitlength * \real{\svgscale}}%
    \fi%
  \else%
    \setlength{\unitlength}{\svgwidth}%
  \fi%
  \global\let\svgwidth\undefined%
  \global\let\svgscale\undefined%
  \makeatother%
  \begin{picture}(1,0.58333333)%
      \put(0,0){\includegraphics[width=\unitlength]{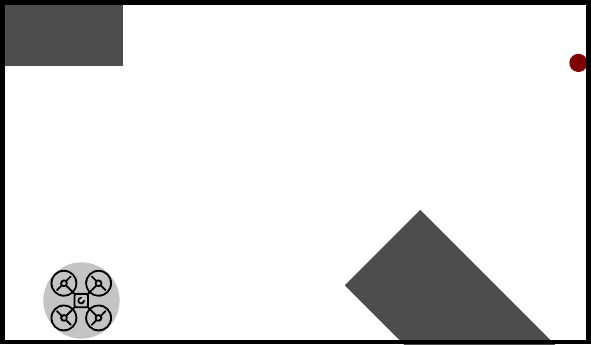}}%
      \put(0,0){\includegraphics[width=\unitlength]{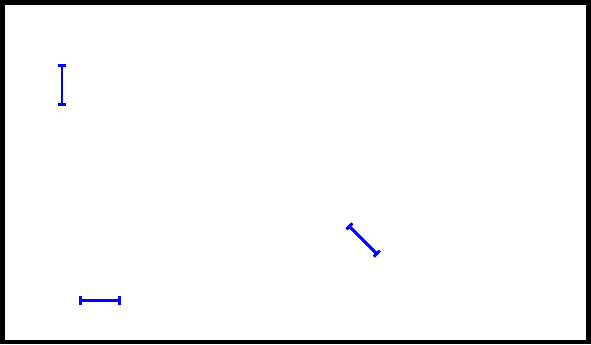}}%
      \put(0,0){\includegraphics[width=\unitlength]{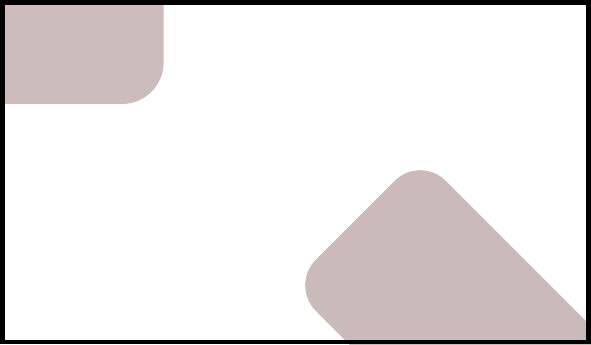}}%
      \put(0,0){\includegraphics[width=\unitlength]{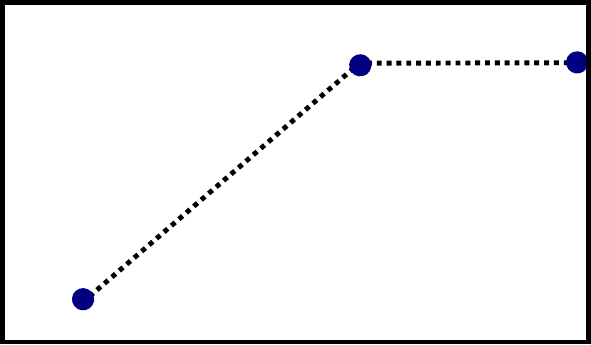}}%
      \put(0,0){\includegraphics[width=\unitlength]{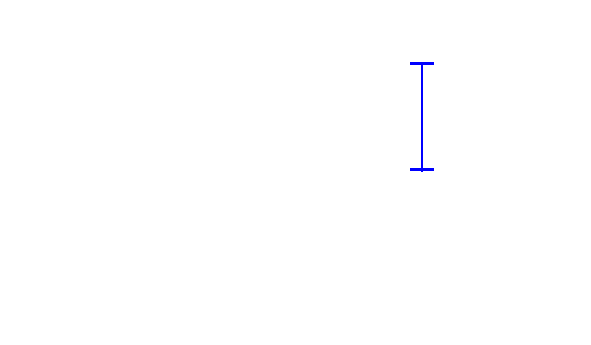}}%
  \footnotesize{
    \put(0.20420342,0.06397075){\color[rgb]{0,0,0}\makebox(0,0)[lb]{\smash{$r$}}}%
    \put(0.005,0.07){\color[rgb]{0,0,0}\makebox(0,0)[lb]{\smash{$\pstart$}}}%
    \put(0.03,0.1){\color[rgb]{0,0,0}\makebox(0,0)[lb]{\smash{$\mathcal{B}$}}}%
    \put(0.89,0.43088928){\color[rgb]{0,0,0}\makebox(0,0)[lb]{\smash{$\pgoal$}}}%
    \put(0.08,0.42){\color[rgb]{0,0,0}\makebox(0,0)[lb]{\smash{$r$}}}%
    \put(0.10,0.35){\color[rgb]{0,0,0}\makebox(0,0)[lb]{\smash{$\chi_{\mathrm{Free}}$}}}
    \put(0.66,0.20){\color[rgb]{0,0,0}\rotatebox{45}{\makebox(0,0)[lb]{\smash{$r$}}}}%
    \put(0.03,0.52){\color[rgb]{1,1,1}\rotatebox{0}{\makebox(0,0)[lb]{\smash{$\chi_{\mathrm{obsReal}}$}}}}%
    \put(0.65,0.13){\color[rgb]{1,1,1}\rotatebox{-45}{\makebox(0,0)[lb]{\smash{$\chi_{\mathrm{obsReal}}$}}}}%
    \put(0.08,0.16){\color[rgb]{0,0,0}\makebox(0,0)[lb]{\smash{$\eta[0]$}}}%
    \put(0.54,0.51){\color[rgb]{0,0,0}\makebox(0,0)[lb]{\smash{$\eta[1]$}}}%
    \put(0.88,0.52){\color[rgb]{0,0,0}\makebox(0,0)[lb]{\smash{$\eta[2]$}}}%
    \put(0.35,0.27){\color[rgb]{0,0,0}\rotatebox{45}{\makebox(0,0)[lb]{\smash{$\eta(\mathbf{s})$}}}}%
    \put(0.73,0.38){\color[rgb]{0,0,0}\makebox(0,0)[lb]{\smash{$\lmin$}}}%
    }  \end{picture}%
\endgroup%
    }
    \hspace{-0.38cm}
    \subfloat[Generation of time-indexed waypoints $\varpi$]{\label{subfig:Waypoints}
    \def\svgwidth{5.95cm}
\begingroup%
  \makeatletter%
  \providecommand\color[2][]{%
    \errmessage{(Inkscape) Color is used for the text in Inkscape, but the package 'color.sty' is not loaded}%
    \renewcommand\color[2][]{}%
  }%
  \providecommand\transparent[1]{%
    \errmessage{(Inkscape) Transparency is used (non-zero) for the text in Inkscape, but the package 'transparent.sty' is not loaded}%
    \renewcommand\transparent[1]{}%
  }%
  \providecommand\rotatebox[2]{#2}%
  \ifx\svgwidth\undefined%
    \setlength{\unitlength}{170.07874016bp}%
    \ifx\svgscale\undefined%
      \relax%
    \else%
      \setlength{\unitlength}{\unitlength * \real{\svgscale}}%
    \fi%
  \else%
    \setlength{\unitlength}{\svgwidth}%
  \fi%
  \global\let\svgwidth\undefined%
  \global\let\svgscale\undefined%
  \makeatother%
  \begin{picture}(1,0.58333333)%
      \put(0,0){\includegraphics[width=\unitlength]{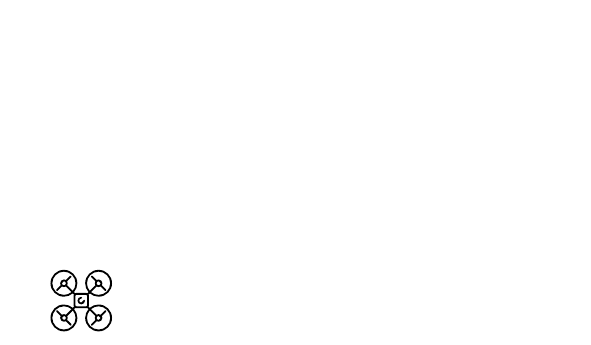}}%
      \put(0,0){\includegraphics[width=\unitlength]{Images/Algorithm2.pdf}}%
      \put(0,0){\includegraphics[width=\unitlength]{Images/Algorithm5.pdf}}%
      \put(0,0){\includegraphics[width=\unitlength]{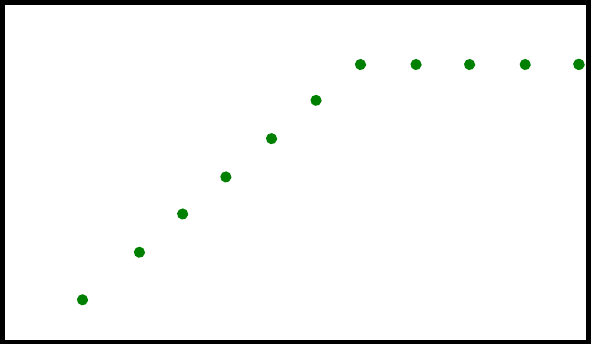}}%
      \put(0,0){\includegraphics[width=\unitlength]{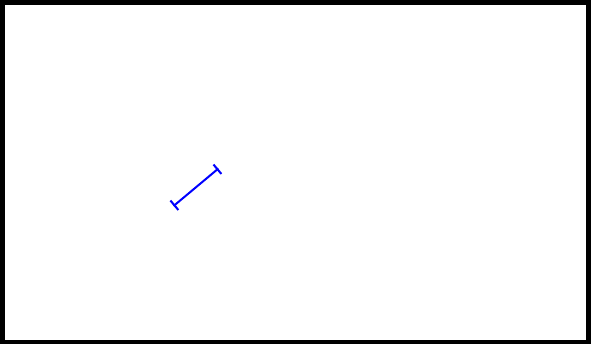}}%
    \footnotesize{
    \put(0.65851852,0.49388888){\color[rgb]{0,0,0}\makebox(0,0)[lb]{\smash{}}}%
    \put(0.2,0.05){\color[rgb]{0,0,0}\makebox(0,0)[lb]{\smash{$\varpi[0],\varpi[1]$}}}%
    \put(0.56,0.4){\color[rgb]{0,0,0}\makebox(0,0)[lb]{\smash{$\varpi[k-1]$}}}%
    \put(0.40,0.49997858){\color[rgb]{0,0,0}\makebox(0,0)[lb]{\smash{$\varpi[k],\varpi[k+1]$}}}%
    \put(0.24,0.12){\color[rgb]{0,0,0}\makebox(0,0)[lb]{\smash{$\varpi[2]$}}}%
    \put(0.31501308,0.26){\color[rgb]{0,0,0}\rotatebox{44.15706771}{\makebox(0,0)[lb]{\smash{$\ell$}}}}%
    
    \put(0.1,0.5){\color[rgb]{0,0,0}\makebox(0,0)[lb]{\smash{$\chi_{\mathrm{obs}}$}}}%
    \put(0.7,0.1){\color[rgb]{0,0,0}\makebox(0,0)[lb]{\smash{$\chi_{\mathrm{obs}}$}}}%
    }
  \end{picture}%
\endgroup%
    }\hspace{-0.38cm}
    \vspace{-0.35cm}
    \subfloat[Generation of soft constraints $\Omega$ for $\varpi$]{\label{subfig:Regions}
    \def\svgwidth{5.95cm}
\begingroup%
  \makeatletter%
  \providecommand\color[2][]{%
    \errmessage{(Inkscape) Color is used for the text in Inkscape, but the package 'color.sty' is not loaded}%
    \renewcommand\color[2][]{}%
  }%
  \providecommand\transparent[1]{%
    \errmessage{(Inkscape) Transparency is used (non-zero) for the text in Inkscape, but the package 'transparent.sty' is not loaded}%
    \renewcommand\transparent[1]{}%
  }%
  \providecommand\rotatebox[2]{#2}%
  \ifx\svgwidth\undefined%
    \setlength{\unitlength}{170.07874016bp}%
    \ifx\svgscale\undefined%
      \relax%
    \else%
      \setlength{\unitlength}{\unitlength * \real{\svgscale}}%
    \fi%
  \else%
    \setlength{\unitlength}{\svgwidth}%
  \fi%
  \global\let\svgwidth\undefined%
  \global\let\svgscale\undefined%
  \makeatother%
  \begin{picture}(1,0.58333333)%
      \put(0,0){\includegraphics[width=\unitlength]{Images/Algorithm12.pdf}}%
      \put(0,0){\includegraphics[width=\unitlength]{Images/Algorithm2.pdf}}%
      \put(0,0){\includegraphics[width=\unitlength]{Images/Algorithm6.pdf}}%
      \put(0,0){\includegraphics[width=\unitlength]{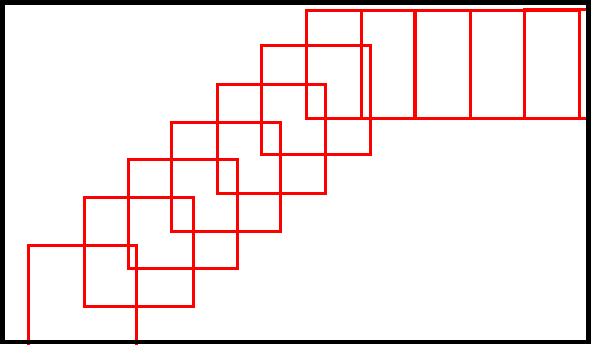}}%
      \put(0,0){\includegraphics[width=\unitlength]{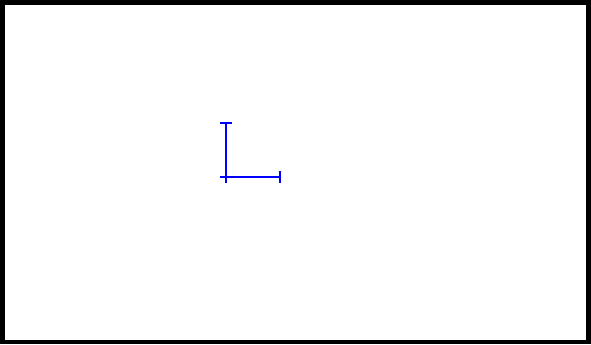}}%
    \footnotesize{
    \put(0.42,0.29){\color[rgb]{0,0,0}\makebox(0,0)[lb]{\smash{$\ell$}}}%
    \put(0.39,0.32){\color[rgb]{0,0,0}\makebox(0,0)[lb]{\smash{$\ell$}}}%
    \put(0.25,0.02){\color[rgb]{0,0,0}\makebox(0,0)[lb]{\smash{$\Omega[0],\Omega[1]$}}}%
    \put(0.24,0.08){\color[rgb]{0,0,0}\makebox(0,0)[lb]{\smash{$\Omega[2]$}}}%
    \put(0.525,0.52){\color[rgb]{0,0,0}\makebox(0,0)[lb]{\smash{$\Omega[k]$}}}%
    }
  \end{picture}%
\endgroup%
    }\hspace{-0.38cm}
    \subfloat[Trajectory Solution, ${p[k]\in\Omega[k]}$]{\label{subfig:TrajectoryD}
    \def\svgwidth{5.95cm}
\begingroup%
  \makeatletter%
  \providecommand\color[2][]{%
    \errmessage{(Inkscape) Color is used for the text in Inkscape, but the package 'color.sty' is not loaded}%
    \renewcommand\color[2][]{}%
  }%
  \providecommand\transparent[1]{%
    \errmessage{(Inkscape) Transparency is used (non-zero) for the text in Inkscape, but the package 'transparent.sty' is not loaded}%
    \renewcommand\transparent[1]{}%
  }%
  \providecommand\rotatebox[2]{#2}%
  \ifx\svgwidth\undefined%
    \setlength{\unitlength}{170.07874016bp}%
    \ifx\svgscale\undefined%
      \relax%
    \else%
      \setlength{\unitlength}{\unitlength * \real{\svgscale}}%
    \fi%
  \else%
    \setlength{\unitlength}{\svgwidth}%
  \fi%
  \global\let\svgwidth\undefined%
  \global\let\svgscale\undefined%
  \makeatother%
  \begin{picture}(1,0.58333333)%
      \put(0,0){\includegraphics[width=\unitlength]{Images/Algorithm2.pdf}}%
      \put(0,0){\includegraphics[width=\unitlength]{Images/Algorithm12.pdf}}%
      \put(0,0){\includegraphics[width=\unitlength]{Images/Algorithm8.pdf}}%
      \put(0,0){\includegraphics[width=\unitlength]{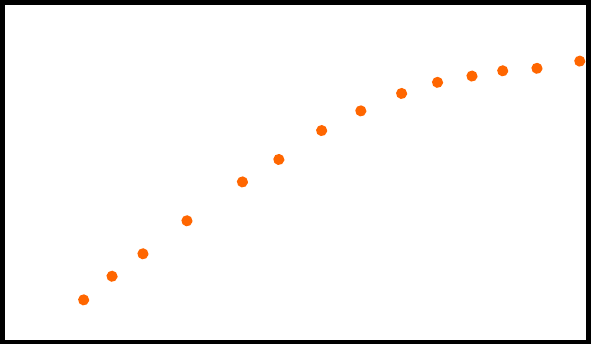}}%
    \footnotesize{
    \put(0.52,0.40){\color[rgb]{0,0,0}\makebox(0,0)[lb]{\smash{$p[k]$}}}%
    \put(0.58,0.45){\color[rgb]{0,0,0}\makebox(0,0)[lb]{\smash{$p[k+1]$}}}%
    }
  \end{picture}%
\endgroup%
    }\hspace{-0.38cm}
    \subfloat[Guaranteed Collision-Free, $p(t)\in\chi_{\mathrm{Free}}$]{\label{subfig:TrajectoryC}
    \def\svgwidth{5.95cm}
\begingroup%
  \makeatletter%
  \providecommand\color[2][]{%
    \errmessage{(Inkscape) Color is used for the text in Inkscape, but the package 'color.sty' is not loaded}%
    \renewcommand\color[2][]{}%
  }%
  \providecommand\transparent[1]{%
    \errmessage{(Inkscape) Transparency is used (non-zero) for the text in Inkscape, but the package 'transparent.sty' is not loaded}%
    \renewcommand\transparent[1]{}%
  }%
  \providecommand\rotatebox[2]{#2}%
  \ifx\svgwidth\undefined%
    \setlength{\unitlength}{170.07874016bp}%
    \ifx\svgscale\undefined%
      \relax%
    \else%
      \setlength{\unitlength}{\unitlength * \real{\svgscale}}%
    \fi%
  \else%
    \setlength{\unitlength}{\svgwidth}%
  \fi%
  \global\let\svgwidth\undefined%
  \global\let\svgscale\undefined%
  \makeatother%
  \begin{picture}(1,0.58333333)%
  \footnotesize{
      \put(0,0){\includegraphics[width=\unitlength]{Images/Algorithm0.pdf}}%
      \put(0,0){\includegraphics[width=\unitlength]{Images/Algorithm2.pdf}}%
      \put(0,0){\includegraphics[width=\unitlength]{Images/Algorithm10.pdf}}%
      \put(0,0){\includegraphics[width=\unitlength]{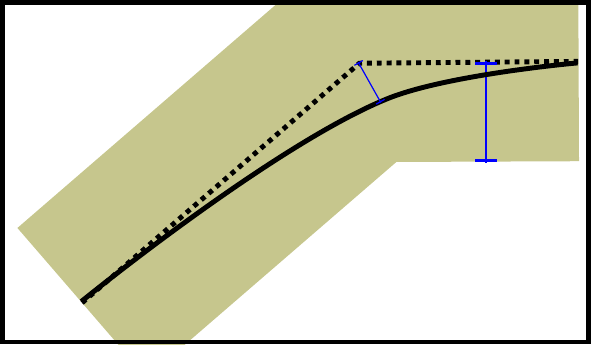}}%
    \put(0.22,0.08){\color[rgb]{0,0,0}\rotatebox{32}{\makebox(0,0)[lb]{\smash{Collision-Free Trajectory $p(t)$}}}}%
    \put(0.23,.17){\color[rgb]{0,0,0}\rotatebox{42}{\makebox(0,0)[lb]{\smash{Obstacle-Free Path $\eta(\mathbf{s})$}}}}%
    \put(0.84,0.37){\color[rgb]{0,0,0}\rotatebox{0}{\makebox(0,0)[lb]{\smash{$\frac{3}{2}\ell\sqrt{d}$}}}}%
    \put(0.65,.42){\color[rgb]{0,0,0}\rotatebox{30}{\makebox(0,0)[lb]{\smash{$b(t)$}}}}%
    }
  \end{picture}%
\endgroup%
    }
    \caption{Illustration of the proposed hybrid method for trajectory planning.}
    \label{fig:TrajPlan}
\end{figure*}

Modeling the center of the robot as a free particle is very effective, for instance, for the trajectory planning of quadcopters, since it has being demonstrated that smooth trajectories for each coordinate can be treated as independent reference outputs due to the differentially flat property~\cite{Loianno2017}.

Our approach to solve Problem~\ref{ProbPlanning} is depicted in~Fig.~\ref{Fig:Approach} and Fig.~\ref{fig:TrajPlan}. The first step, illustrated in~Fig.~\ref{subfig:Path}, is to generate an obstacle-free path $\eta(\mathbf{s})\subset\chi_{\mathrm{Free}}$ connecting $\pstart$ with $\pgoal$; where $\eta(\mathbf{s})\subset\chi_{\mathrm{Free}}$. High-clearance paths are desired in order to avoid over-constraining the optimization problem. Hence, a sampling-based path planning technique capable of efficiently generating such paths should be used (e.g.~\cite{Jaillet2010,Devaurs2016}). The minimum distance $\lmin$ from the path $\eta(\mathbf{s})$ to the obstacle region $\chi_\mathrm{obs}$ is used to set a design parameter $\ell$.
If a minimum separation $\lmin$ is required then the obstacles can be further inflated by $\lmin$ before the path planning algorithm is run.

The next step is to generate a series of time-indexed waypoints $\varpi[k]$ with associated hypercube regions $\Omega[k]$ of edge length $2\ell$, as shown in Fig.~\ref{subfig:Waypoints} and Fig.~\ref{subfig:Regions}. These regions will be used later to introduce soft constraints in the formulation of the optimization problem. Note that in addition to the waypoints inserted at regular intervals along the path, an extra waypoint is added in each node $\eta[s]$. While these extra waypoints may be redundant in some cases, i.e., where the adjacent path segments don't present a drastic change of direction, they are used to guarantee existence of a trajectory solution. As illustrated in Fig.~\ref{subfig:TrajectoryD}, it is required that such a solution satisfies $p[k]=p(hk)\in\Omega[k]$; where the time step $h$ is chosen as a function of $\ell$ and the maximum acceleration $A_{\mathrm{max}}$ of the robot. Let $K+1$ be the number of waypoints $\varpi$, then $t_f=Kh$ is the total time of the trajectory.

It is shown that following our formulation, the posed quadratic optimization problem is guaranteed to be feasible.
Furthermore, let the separation $b(t)$ between $\eta(\mathbf{s})$ and $p(t)$ be given by 
\begin{align*}
b(t)=\min\{\|p(t)-\lambda (\eta[s+1]-&\eta[s])+\eta[s]\}\|_2:\\
&\lambda\in[0,1], s\in\{0,\ldots,S-1\} \}.    
\end{align*}
The proposed method ensures that $b(t)$ is bounded by $\frac{3}{2}\ell\sqrt{d}$, then $p(t)\in\chi_{\mathrm{Free}}$ for all $t\in[0,t_f]$, as illustrated in~Fig.~\ref{subfig:TrajectoryC}. This is also an important contribution since previous proposals reported in the literature either do not guarantee to be collision-free, see e.g.~\cite{Kalakrishnan2011,Zucker2013} or only guarantee that a discrete trajectory generated by the \textit{QP} formulation is collision-free but not the interpolated trajectory that is actually executed by the robot, see e.g.~\cite{Augugliaro2012}. Moreover, they require iterative \textit{QP} formulations with decreasing sampling time.

In many important situations, it is necessary to compute a new trajectory before the robot has completed its current plan. For example, a path that was previously open may have become blocked, or a new task may require the robot to move towards a different final destination. Hence, as depicted in~Fig.~\ref{Fig:Approach} while the robot is executing the last solution of Problem~\ref{ProbPlanning} it should continuously check if a change of plan is required. 
Let $\bar{t}$ be the current time. If the conditions to re-plan are met then a new optimization problem is formulated with initial state $\xstart\leftarrow x(\bar{t}+t_s)$ and final conditions $\xgoal$, and a new trajectory is generated as described above. Otherwise, the trajectory in the interval $[\bar{t},\bar{t}+t_s]$ is committed to the trajectory tracking controller while the path related to the interval $[\bar{t}+t_s,t_f]$ is further refined, for a limited number of iterations, in search of improving the cost or the clearance of the solution, using asymptotically optimal methods~\cite{Karaman2011}.

\section{Preliminaries: Sampling-based Path Planning}\label{sec:path_planing}

Sampling based, or stochastic search methods to generate obstacle-free paths~\cite{Gammell2015,Choudhury2016,Elbanhawi2014}, are popular due to their effectiveness at finding traversable paths while avoiding discretization of the state space. In this paper, collision-free path generation is illustrated using Informed Optimal Rapidly-exploring Random Trees (\textit{IRRT*})~\cite{Gammell2014}. This algorithm improves the current solution iteratively via incremental rewiring, converging asymptotically to an optimal solution. While any other technique to generate a path could have been used, such as \textit{BIT*}~\cite{Gammell2015} or Regionally Accelerated \textit{BIT*} (\textit{RABIT*})~\cite{Choudhury2016}, the \textit{IRRT*} algorithm was selected because it allows anytime planning by limiting its solutions~\cite{Karaman2011Anytime}. A brief review of the \textit{IRRT*} algorithm~\cite{Gammell2014} is presented in this section.

Assume that the initial position $\pstart \in \chi_\mathrm{free}$ and the final position $\pgoal \in \chi_\mathrm{free}$ are given. The aim is to find a path $\eta^*$, that minimizes a given cost function $c:\Sigma(\eta)\mapsto \mathbb{R}_{\geq0}$, where $\Sigma(\eta)$ is the set of all feasible paths, i.e.
$\eta^* = \operatorname*{arg\,min}_{\eta\in\Sigma(\eta)} c(\eta).$ Following the standard \textit{RRT*} algorithm~\cite{Karaman2011,Elbanhawi2014} a new random node in $\chi_{\mathrm{Free}}$ is sampled and its neighbor nodes within a ball of radius $r_{\mathrm{RRT}^*}$ are rewired if this provides a path with lower cost function, where
$r_{\mathrm{RRT}^*} = \gamma_{\mathrm{RRT}^*}\left (\frac{\log(\delta )}{\delta} \right )^\frac{1}{d}$; $\delta$ is the number of states in the tree and $\gamma_{\mathrm{RRT}^*}$ an appropriate constant given in~\cite{Karaman2011}. This paper seeks to minimize the path length in $\mathbb{R}^d$, so the cost function $c$ is based on the Euclidean distance.

After a path has been obtained using \textit{RRT*}, the \textit{IRRT*} algorithm allows increasing the probability of reducing the cost of the solution in subsequent rounds of \textit{RRT*} by restricting the sampling region. In~\cite{Gammell2014} it is shown that the new nodes that may improve the cost of the path are necessarily contained in the ellipse defined by
\begin{align*}
\chi_\mathrm{inform} = \{\mathrm{x} \in \chi \quad : \quad \|\pstart - \mathrm{x}\|_2+\|\mathrm{x} - \pgoal\|_2 \leq c_\mathrm{best}\}, 
\end{align*}
where $c_\mathrm{best}$ is the cost of the current solution. New nodes sampled from $\chi_\mathrm{inform}$ are generated by taking random samples $\mathrm{x_{ball}}$ from a unitary $d$-ball and then mapped to $\chi_\mathrm{inform}$ using
$\mathrm{\mathrm{x}_\mathrm{ellipse}} = \mathbf{CL}\mathrm{x_{ball}}+\mathrm{x_{centre}}$, where \begin{align*}
\mathbf{L} = \mathrm{diag}\bigg\{ 
\frac{c_\mathrm{best}}{2},\overbrace{\frac{(c^2_\mathrm{best}-c^2_\mathrm{min})^{1/2}}{2},\cdots,\frac{(c^2_\mathrm{best}-c^2_\mathrm{min})^{1/2}}{2}}^{d-1}\bigg\},
\end{align*} 
is a diagonal matrix, $\mathbf{C} \in SO(d)$ is the rotation between the origin of the $d$-dimensional Euclidean space and the vector $\pgoal - \pstart$ and $\mathrm{x_{centre}}=(\pgoal+\pstart)/2$.

\section{Optimization-based Trajectory Planning}\label{sec:modelbased_optimization}

The \textit{QP} formulation derived in this section integrates information from the obstacle-free path $\eta(\mathbf{s})$ together with knowledge of the dynamical constraints of the robot. The problem is posed in such a way that the need for iterative solutions is avoided and the resulting trajectory is guaranteed to be collision-free, i.e. $\forall t$, $p(t)\in\chi_{\mathrm{Free}}$. 

To guide the solution of the \textit{QP} and to ensure that the separation $b(t)$ between the planned path $\eta(\mathbf{s})$ and the final trajectory $p(t)$ remains adequately bounded, $K+1$ time-indexed waypoints $\varpi[k]$ are placed along the path $\eta(\mathbf{s})$. Each waypoint $\varpi[k]$ has an associated set of soft constraints defining a hypercube region $\Omega[k]$, as illustrated in Fig.~\ref{subfig:Regions}. These constraints lead to $p[k]=p(kh)\in\Omega[k]$, $k=0,\ldots,K$. The time step $h$ and the maximum velocity $V_\mathrm{max}$ are obtained based on the value of the parameter $\ell$ and the maximum acceleration $A_\mathrm{max}$. This method ensures that a trajectory $p(t)$ can always be found which travels from one hypercube region to the next in time $h$ while the separation $b(t)$ remains bounded by $\frac{3}{2}\ell\sqrt{d}$, as illustrated in Fig.~\ref{Fig:Lemmas}. Since there is a minimum clearance of $\lmin$, this property guarantees a collision-free trajectory. In the following, this procedure is explained in greater detail.

\begin{figure}
\begin{center}
	\def\svgwidth{7cm} 
\begingroup%
  \makeatletter%
  \providecommand\color[2][]{%
    \errmessage{(Inkscape) Color is used for the text in Inkscape, but the package 'color.sty' is not loaded}%
    \renewcommand\color[2][]{}%
  }%
  \providecommand\transparent[1]{%
    \errmessage{(Inkscape) Transparency is used (non-zero) for the text in Inkscape, but the package 'transparent.sty' is not loaded}%
    \renewcommand\transparent[1]{}%
  }%
  \providecommand\rotatebox[2]{#2}%
  \ifx\svgwidth\undefined%
    \setlength{\unitlength}{129.46928129bp}%
    \ifx\svgscale\undefined%
      \relax%
    \else%
      \setlength{\unitlength}{\unitlength * \real{\svgscale}}%
    \fi%
  \else%
    \setlength{\unitlength}{\svgwidth}%
  \fi%
  \global\let\svgwidth\undefined%
  \global\let\svgscale\undefined%
  \makeatother%
  \begin{picture}(1,0.66186112)%
    \put(0,0){\includegraphics[width=\unitlength]{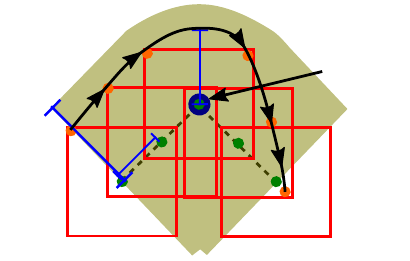}}%
    \footnotesize{
    \put(0.25,0.27){\color[rgb]{0,0,0}\rotatebox{45}{\makebox(0,0)[lb]{\smash{$\frac{2}{3}\ell\sqrt{d}$}}}}%
    \put(0.32,0.26){\color[rgb]{0,0,0}\rotatebox{45}{\makebox(0,0)[lb]{\smash{$\ell$}}}}%
    \put(0.79,0.45){\color[rgb]{0,0,0}\makebox(0,0)[lb]{\smash{$\varpi[k],\varpi[k+1]$}}}%
    \put(0.5,0.45){\color[rgb]{0,0,0}\makebox(0,0)[lb]{\smash{$b(t)$}}}%
    \put(0.47,0.58){\color[rgb]{0,0,0}\makebox(0,0)[lb]{\smash{$p(t)$}}}%
    \put(0.33,0.54){\color[rgb]{0,0,0}\makebox(0,0)[lb]{\smash{$p[k]$}}}%
    \put(0.74,0.15){\color[rgb]{0,0,0}\makebox(0,0)[lb]{\smash{$p[k+3]$}}}%
    \put(0.05,0.3){\color[rgb]{0,0,0}\makebox(0,0)[lb]{\smash{$p[k-2]$}}}%
    \put(0.62,0.53){\color[rgb]{0,0,0}\makebox(0,0)[lb]{\smash{$p[k+1]$}}}%
    \put(0.2,0.07){\color[rgb]{0,0,0}\makebox(0,0)[lb]{\smash{$\Omega[k-2]$}}}%
    \put(0.15,0.18){\color[rgb]{0,0,0}\makebox(0,0)[lb]{\smash{$\varpi[k-2]$}}}%
    \put(0.68,0.07){\color[rgb]{0,0,0}\makebox(0,0)[lb]{\smash{$\Omega[k+3]$}}}%
    \put(0.72,0.18){\color[rgb]{0,0,0}\makebox(0,0)[lb]{\smash{$\varpi[k+3]$}}}%
    }
  \end{picture}%
\endgroup%
\end{center}
\caption{
If $V_{\mathrm{max}}$ and $h$ are a solution to Problem~\ref{Prob:h}, then the robot takes time $h$ to navigate from one point in the discrete trajectory to the next (i.e. $p(kh)=p[k]$, $k=0,\ldots,K$) while the separation $b(t)$ between the resulting trajectory $p(t)$ and the continuous path $\eta(\mathrm{s})$ is bounded by $\frac{2}{3}\ell\sqrt{d}$.
}
\label{Fig:Lemmas}
\end{figure}
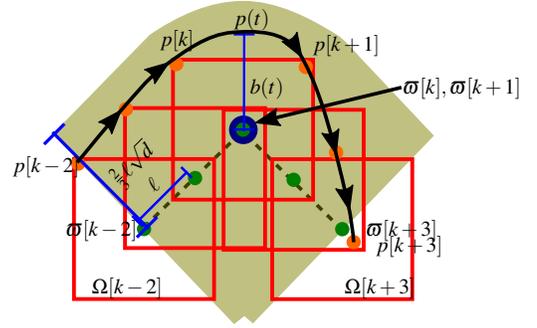

\subsection{Problem constraints formulation}
\label{Sec:Constrains}

Let $\kappa_0=\kappa_S=0$ and $\kappa_{s}=\left\lceil \frac{1}{\ell} \|\eta[s+1]-\eta[s]\|_2 \right\rceil$, $s\in\{1,\ldots,S-1\}$. A sequence of waypoints are obtained such that  $\varpi_0[\kappa_0]=\eta[0]$, $\varpi_S[\kappa_S]=\eta[S]$ and $
\varpi_s\left[\textup{i}\right]=\frac{\textup{i}}{\kappa_{s}}\left(\eta[s+1]-\eta[s]\right)+\eta[s]
$, $\textup{i}\in\{0,\ldots, \kappa_s\}$. The time-indexed waypoint $\varpi[k]$ is obtained as the $(k-1)-$th element of the sequence
\begin{align*}
\varpi=\{\varpi_0[\kappa_0],\varpi_1[0],\ldots,&\varpi_1[\kappa_1],\ldots,\\
&\varpi_{S-1}[0],\ldots,\varpi_{S-1}[\kappa_{S-1}],\varpi_S[\kappa_S]\}.
\end{align*}
The generation of time-indexed waypoints is illustrated in~Fig.~\ref{subfig:Waypoints}; notice that $\varpi_{s-1}[\kappa_{s-1}]=\varpi_{s}[0]=\eta[s]$ or equivalently $\varpi[k] = \varpi[k+1] = \eta[s]$ for some $k$. Thus, in addition to the generation of waypoints in-between two consecutive nodes, this procedure also adds two waypoints at each node $\eta[s]$ of the path $\eta$, as illustrated in Fig.~\ref{subfig:Waypoints} and Fig.~\ref{Fig:Lemmas} . Furthermore, as illustrated in~Fig.~\ref{subfig:Regions}, each waypoint $\varpi[k]$ has an associated hypercube of edge length $2\ell$ defined by
\begin{equation}
\Omega[k]=\{\rho: \|\rho-\varpi[k]\|_\infty\leq \ell \}.
\label{Eq:WayPointRegion}
\end{equation}

The $\Omega[k]$ regions are used as constraints on the convex optimization problem restricting a solution of the trajectory $p(t)$ to be one satisfying $p[k]=p(kh)\in\Omega[k]$, $k\in\{0,\ldots,K\}$, (or equivalently $\|\varpi[k]-p[k]\|_\infty\leq \ell$) as illustrated in~Fig.~\ref{subfig:TrajectoryD}; where $h$ is the time step hereinafter derived and $K=|\varpi|-1$. Hence, if such constraints are satisfied, the total time of the trajectory along the path $\eta$ is $t_f=Kh$.

\begin{problem}
\label{Prob:h}
Based on the design parameter $\ell$, the maximum acceleration $A_{\mathrm{max}}$ and taking into account particle kinematics for motion with constant acceleration $\|a[k]\|_\infty \leq A_{\mathrm{max}}$ in $t\in(kh,(k+1)h] $, $k\in{0,\cdots,K}$, find $\vmax$ and $h$ such that the following conditions are satisfied:
\begin{enumerate}[i)]
\item \label{Cond_1} Let $\|\varpi[k+1]-\varpi[k]\|_2=\ell$ and assume that $p[k] \in \Omega[k]$ and $v[k]$ satisfies $\| v[k]\|_\infty \leq \vmax$ and $v[k]^\trans(\varpi[k+1]-\varpi[k])\geq 0$. Then, there exists a constant acceleration $ \|a[k]\|_\infty\leq\amax$ such that $ p[k+1] \in \Omega[k+1]$, $ \| v[k+1]\|_\infty\leq \vmax$ and $ v[k+1]^\trans(\varpi[k+1]-\varpi[k])\geq 0$.
\item \label{Cond_2} Let $\varpi[k]=\varpi[k+1]=\eta[s]$ and assumme that $ p[k] \in \Omega[k] $, and $\|v[k]\|_\infty \leq \vmax$. Then, there exists a pair of constant accelerations  $\|a[k]\|_\infty\leq\amax$ and  $\|a[k+1]\|_\infty\leq\amax$,  such that $ p[k+1] \in \Omega[k+1] $, $ \|v[k+1]\|_\infty \leq\vmax$, $ p[k+2] \in \Omega[k+2]$ and $ \|v[k+2]\|_\infty\leq\vmax$ with $v[k+2]^\trans(\varpi[k+2]-\varpi[k+1])\geq 0$.
\item \label{Cond_3} For $p[k]\in\Omega[k]$ and $\|v[k]\|_\infty\leq\vmax$ and with $\|a[k]\|_\infty\leq \amax$ such that $p[k+1]\in\Omega[k+1]$ and $\|v[k+1]\|_\infty\leq\vmax$, it holds that for $t\in[kh,(k+1)h]$, the separation $b(t)$ between $p(t)$ and the straight line connecting $\eta[s]$ and $\eta[s+1]$ is bounded by $|b(t)|\leq\frac{3}{2}\ell\sqrt{d}$.
\end{enumerate}
\end{problem}

 Assuming condition~\ref{Cond_1}) is satisfied implies that the particle can move in the direction of the path from $p[k-2]\in\Omega[k-2]$ to $p[k-1]\in\Omega[k-1]$ satisfying the constraints on velocity and acceleration. Condition~\ref{Cond_2}) applies when the particle moves between $p[k]$ and $p[k+1]$ while it undergoes a sharp change of direction in the path. If this condition is satisfied then a trajectory satisfying the constraints on velocity an acceleration can change direction in a time $h$ and continue satisfying such constraints. To account for such cases the formulation places two identical waypoints at the path node $\eta[s]$, i.e. $\varpi[k]=\varpi[k+1]=\eta[s]$. Finally, condition~\ref{Cond_3}) implies that $p(t)$ is also collision free since $b(t)$ is bounded in such way that $p(t)\in\chi_\mathrm{Free}$, $\forall t\in[0,t_f]$. 

It is shown, in Lemma~\ref{lemma:2l}--Lemma~\ref{lemma:bbar} in the Appendix, that a selection of $\vmax$ and $h$ solving Problem~\ref{Prob:h} is 
\begin{align}
         \vmax^2 := \ell\amax, \ \
         h^2 := &\frac{4\ell}{\amax}. \label{def_h}
\end{align}
It follows that there exists a collision-free trajectory $p(t)$ with duration $t_f$ satisfying $p[k]\in\Omega[k]$ and the dynamical constraints $\|v[k]\|_\infty\leq\vmax$, $\|a[k]\|_\infty\leq\amax$ such that the separation $b(t)$ between $p(t)$ and $\eta(\mathbf{s})$ is bounded by $\frac{3}{2}\ell\sqrt{d}$.
\begin{theorem}
\label{prop:Time}
Let $\chi_{\mathrm{Free}}$ be the obstacle free space,  $\eta$ be an obstacle-free path connecting $\pstart$ to $\pgoal$ and $\vmax,h$ defined by~\eqref{def_h}. Then given $\ell$, there exist a trajectory $p(t)$ such that $\forall t\in[0,t_f]$ where $t_f=Kh=(|\varpi|-1)h$, the following are satisfied:
\begin{itemize}
    \item $|b(t)|\leq\frac{3}{2}\ell\sqrt{d}$
    i.e. $p(t)\in\chi_{\mathrm{Free}}$
    \item $\|v[k]\|_\infty\leq\vmax$
    \item $\|a[k]\|_\infty\leq\amax$
    \item $p(0)=\pstart$, $v(0)=\vstart$, $v(t_f)=\vgoal$ for given $\pstart$,$\vstart$ and $\vgoal$
\end{itemize}
\end{theorem}
\begin{proof}
The proof follows from Lemma~\ref{lemma:2l} -- Lemma~\ref{lemma:bbar}.
\end{proof}

\subsection{Quadratic Program Formulation}
In the following a \textit{QP} to produce a trajectory for the robot is posed taking into account its dynamical constraints. This formulation, is based of the constraints derived in Section~\ref{Sec:Constrains}.

The objective function for the convex problem is formulated as a function of the acceleration by
\begin{align}\label{eq:objetivefunction}
\min_\mathbf{a} \quad \mathbf{a}^T H \mathbf{a},
\end{align}
where $\mathbf{a}=[a[0]^T \ \cdots \ a[K]^T ]^T\in \mathbb{R}^{dK}$ is a vector of accelerations with $a[k]\in \mathbb{R}^d$ as the acceleration at the discretization time $k$ and the matrix $H\in \mathbb{R}^{dK\times dK}$ is a positive definite constant matrix. By selecting $H$ appropriately a quadratic objective function of $(p,v,a,j)$ can be obtained~\cite{Chen2015decoupled}. For instance, since jerk is given by
$j[k]=\frac{a[k+1]-a[k]}{h}
$
by taking $H=W^T W$ 
$$
[w]_{\textup{ij}}=\left\lbrace
\begin{array}{rl}
    -1/h & \text{if } \textup{i}=\textup{j},  \\
     1/h & \text{if } \textup{i}=\textup{j}-1, \\
     0 & \text{otherwise}
\end{array}
\right.
$$
we obtain an objective function for minimizing jerk. Thus, the following convex optimization problem is defined:
\begin{problem}
\label{Prob:QP}
Taking $p[k]$ and $v[k]$ as 
\begin{align*}
 p[k]=&p[0]+hkv[0] + \frac{h^2}{2}\sum_\textup{i=0}^{k-1} (2(k-\textup{i})+3)a[\textup{i}] \\
v[k]=&v[0]+h\sum_\textup{i=0}^{k-1} a[\textup{i}],
\end{align*}
obtain a sequence of accelerations $a=[a[0]^T,\ldots,a[K]^T]^T$ that
$$
\text{ minimizes } \ \ \ \mathbf{a}^T H \mathbf{a}
$$
subject to
\begin{align*}
\begin{matrix}
p[0]=\pstart,  & v[0]=\vstart, & a[0]=\astart, \end{matrix}
\end{align*}
\begin{align*}
\begin{matrix}
\|\varpi[k]-p[k]\|_\infty\leq\ell, & \|v[k]\|_\infty\leq\vmax, & \|a[k]\|_\infty\leq\amax
\end{matrix}
\end{align*}
for $k=1,\ldots,K-1$ and
\begin{align*}
\begin{matrix}
p[K]=\pgoal, & v[K]=\vgoal, & a[K]=\agoal.
\end{matrix}
\end{align*}
Notice that $\|p[k]-\varpi[k]\|_\infty\leq\ell$ is equivalent to $p[k]\in\Omega[k]$.
\end{problem}

Since $p[k]$ and $v[k]$ are linear functions of $a[k]$ the optimization problem given in Problem~\ref{Prob:QP} is convex. Moreover, notice that by~Theorem~\ref{prop:Time} the posed \textit{QP} formulation is feasible. This approach, unlike~\cite{Augugliaro2012,Neunert2016,Richter2016,Chen2015decoupled}, avoids iterative formulations. 

\section{Results}\label{sec:results}
 
The proposed algorithm was programmed in C++ and implemented on a PC running Ubuntu 14.04-LTS with Intel Core i5-3210M @ 2.50GHz and 4~GB of RAM. The \textit{QP} is solved using Mosek~\cite{Mosek}. First, benchmarking results are shown comparing our method with state-of-the-art algorithms based on simulations in various forest environments generated using the method in~\cite{KaramanForest}. Then, we present experimental results obtained with a \textit{Crazyflie 2.0}~\cite{CrazyFlie} quadcopter in different three dimensional scenarios with varying number of obstacles and complexity - see video at \url{https://youtu.be/DJ1IZRL5t1Q}.

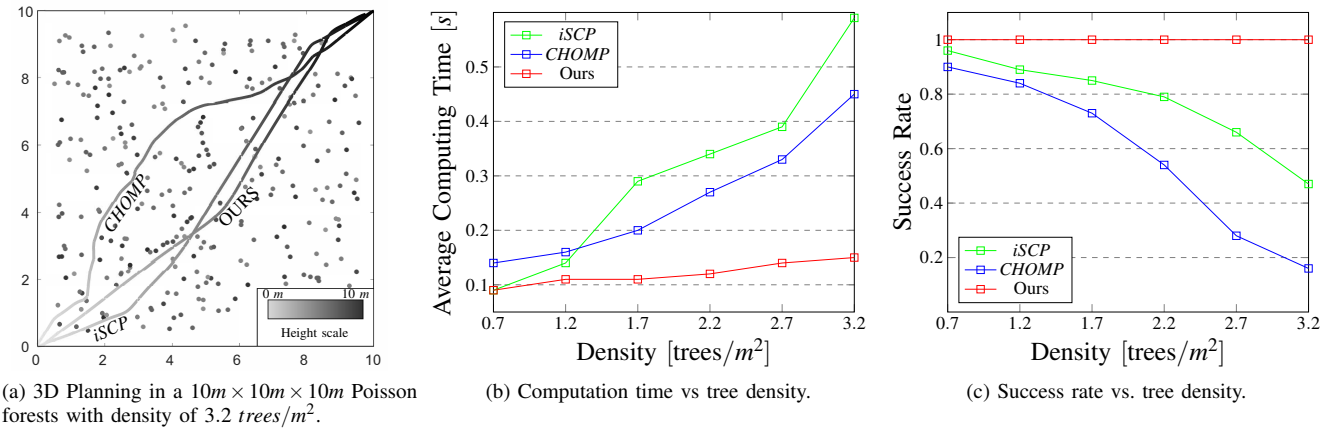
\begin{figure*}
     \centering
     \subfloat[3D Planning in a $10m \times 10m \times 10m $ Poisson forests with density of $3.2~trees/m^2$.]{\label{subfig:CompA}
     \begin{picture}(150,1)
     \put(0,0){\includegraphics[width=5.0cm]{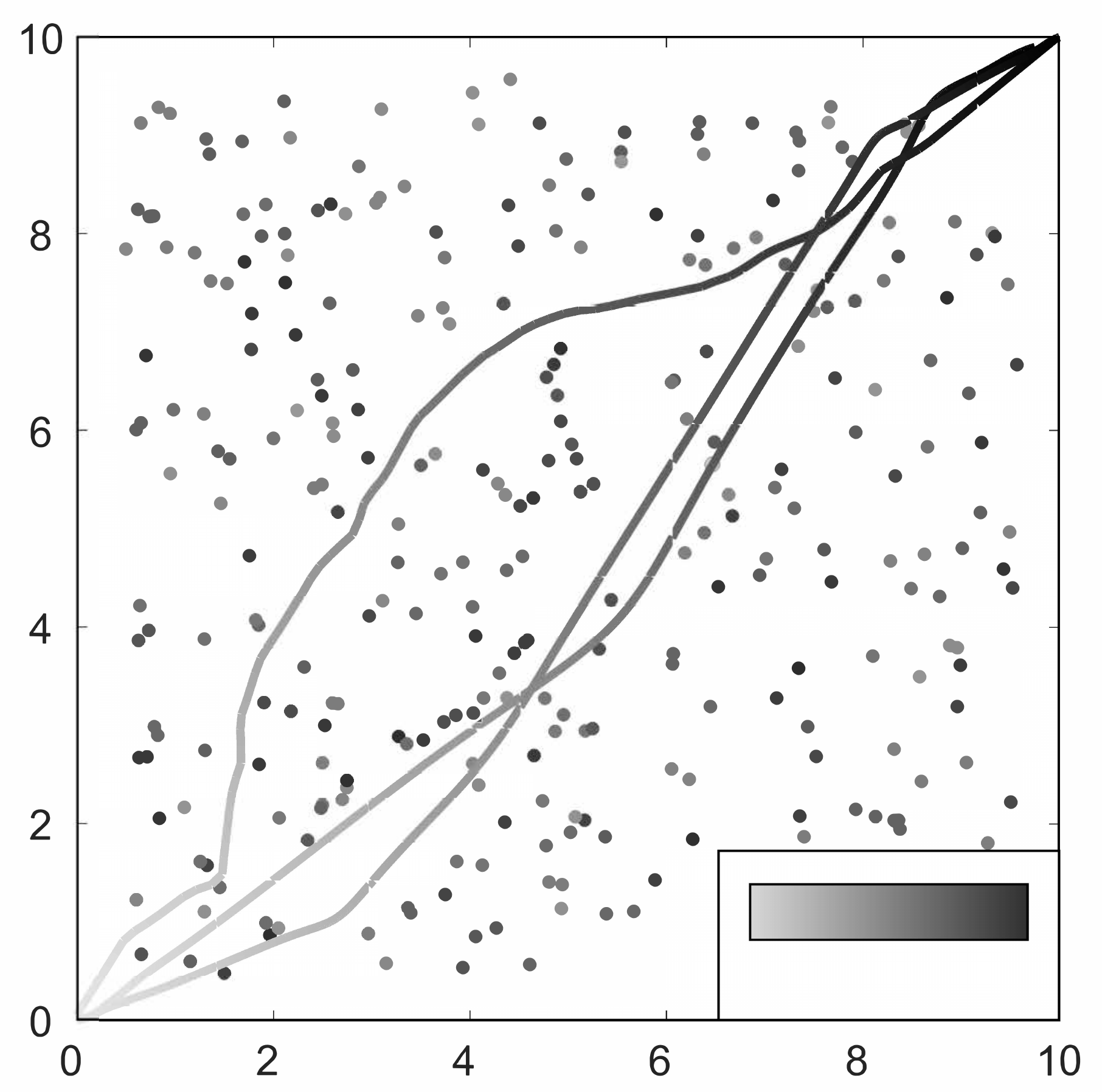}}
     \scriptsize{
     \put(95,27.5){\color[rgb]{0,0,0}\rotatebox{0}{\makebox(0,0)[lb]{\smash{\tiny{$0~m$}}}}}%
     \put(125,27.5){\color[rgb]{0,0,0}\rotatebox{0}{\makebox(0,0)[lb]{\smash{\tiny{$10~m$}}}}}%
     \put(38,52){\color[rgb]{0,0,0}\rotatebox{55}{\makebox(0,0)[lb]{\smash{\textit{CHOMP}}}}}%
     \put(32,11){\color[rgb]{0,0,0}\rotatebox{25}{\makebox(0,0)[lb]{\smash{\textit{iSCP}}}}}%
     \put(81,54){\color[rgb]{0,0,0}\rotatebox{45}{\makebox(0,0)[lb]{\smash{OURS}}}}%
     \put(102,13){\color[rgb]{0,0,0}\rotatebox{0}{\makebox(0,0)[lb]{\smash{\tiny{Height scale}}}}}%
     }
     \end{picture}
     }
     \subfloat[Computation time vs tree density.]{\label{subfig:CompB}
     \begin{tikzpicture}[thick,scale=0.7, every node/.style={font=\Large, scale=0.7}]
\begin{axis}[
    xlabel={Density $[\mathrm{trees}/m^2]$ },
    ylabel={Average Computing Time $[s]$ },
    xmin=0.7, xmax=3.2,
    ymin=0.05, ymax=0.60,
    xtick={0.7,1.2,1.7,2.2,2.7,3.2},
    ytick={0.10,0.20,0.30,0.40,0.50},
    yticklabel style = {font=\Large},
    xticklabel style = {font=\Large},
    label style = {font=\huge},
    legend pos=north west,
    ymajorgrids=true,
    grid style=dashed,
]
 \addplot[
    color=green,
    mark=square,
    ]
    coordinates {
    (0.7,0.09)(1.2,0.14)(1.7,0.29)(2.2,0.34)(2.7,0.39)(3.2,0.59)
    };

 \addplot[
    color=blue,
    mark=square,
    ]
    coordinates {
    (0.7,0.14)(1.2,0.16)(1.7,0.20)(2.2,0.27)(2.7,0.33)(3.2,0.45)
    };
 
\addplot[
    color=red,
    mark=square,
    ]
    coordinates {
    (0.7,0.09)(1.2,0.11)(1.7,0.11)(2.2,0.12)(2.7,0.14)(3.2,0.15)
    };
    
    \addlegendentry{\textit{iSCP}}
    \addlegendentry{\textit{CHOMP}}
    \addlegendentry{Ours}
\end{axis}
\end{tikzpicture}
     }
     \subfloat[Success rate vs. tree density.]{\label{subfig:CompC}
     \begin{tikzpicture}[thick,scale=0.7, every node/.style={font=\Large, scale=0.7}]
\begin{axis}[
    xlabel={Density $[\mathrm{trees}/m^2]$ },
    ylabel={Success Rate},
    xmin=0.7, xmax=3.2,
    ymin=0.0, ymax=1.1,
    xtick={0.7,1.2,1.7,2.2,2.7,3.2},
    ytick={0.2,0.4,0.6,0.8,1.0},
    legend pos=south west,
    ymajorgrids=true,
    grid style=dashed,
    yticklabel style = {font=\Large},
    xticklabel style = {font=\Large},
    label style = {font=\huge},
]
 \addplot[
    color=green,
    mark=square,
    ]
    coordinates {
    (0.7,0.96)(1.2,0.89)(1.7,0.85)(2.2,0.79)(2.7,0.66)(3.2,0.47)
    };

 \addplot[
    color=blue,
    mark=square,
    ]
    coordinates {
    (0.7,0.90)(1.2,0.84)(1.7,0.73)(2.2,0.54)(2.7,0.28)(3.2,0.16)
    };
 
\addplot[
    color=red,
    mark=square,
    ]
    coordinates {
    (0.7,1.0)(1.2,1.0)(1.7,1.0)(2.2,1.0)(2.7,1.0)(3.2,1.0)
    };
    
    \addlegendentry{\textit{iSCP}}
    \addlegendentry{\textit{CHOMP}}
    \addlegendentry{Ours}
\end{axis}
\end{tikzpicture}
     }
     \caption{Algorithm comparison among \textit{iSCP}, \textit{CHOMP} and the proposed approach on different forest environments varying on tree density. }
    \label{fig:Comp}
\end{figure*}

\subsection{Benchmark results}

    The algorithm was compared against the Incremental \textit{SCP} (\textit{iSCP}) algorithm described in \cite{Chen2015decoupled} and against \textit{CHOMP}~\cite{Zucker2013}. Simulations were run in different Poisson forest environments~\cite{KaramanForest} with varying tree densities in a simulation space of $10m \times 10m \times 10m$. The height of the trees was set following a uniform distribution between $5m$ to $10m$. The start and goal positions for the mobile robot were selected randomly, such that the minimum travel distance was $8m$ starting and ending at rest with a maximum acceleration of $20 m/s^2$. For each benchmark 500 tests were performed.

\begin{table}[]
\centering
\begin{tabular}{c|cccc}
\textbf{Algorithm} & \textbf{\begin{tabular}[c]{@{}c@{}} Success\\ Rate\end{tabular}} & \textbf{\begin{tabular}[c]{@{}c@{}}Avg. Path\\ Length {[}m{]}\end{tabular}} & \textbf{\begin{tabular}[c]{@{}c@{}}Avg. Max\\ Velocity {[}m/s{]}\end{tabular}} & \textbf{\begin{tabular}[c]{@{}c@{}}Avg. Comp.\\ Time {[}s{]}\end{tabular}} \\ \hline
\textit{iSCP}  & 239/500 & \textbf{11.6886}     &  1.4612    &  0.4920 \\
\textit{CHOMP}      & 83/500  & 14.9088 & \textbf{1.7677} & 0.4495 \\
Ours  &   \textbf{500/500}   & 15.8581  &   1.5266    & \textbf{0.1519}                                   
\end{tabular}
\caption{Comparison in a Poisson forest with a density of 3.2 $\mathrm{trees}/m^2$.}
\label{Tab:results}
\end{table}

Some parameters had to be selected for each algorithm. The \textit{iSCP} was configured using a total time of 15 seconds and a discretization step of 0.2 seconds. In the case of \textit{CHOMP} it was configured using a fixed value $N$ of 100 points. These parameters were selected trying to balance the success rate and the computation time. They were set based on trial and error since no methodology is known to calculate them in advance. Finally, for the proposed algorithm, a minimum separation to the obstacles is required, as described in Section~\ref{sec:outline}. Thus, the obstacles were inflated by $\lmin$, resulting in $\ell = 0.05$. For path generation a goal region with a radius of $0.005m$ was specified. The \textit{IRRT*} algorithm was run at most four rounds of path generation followed by refinement of the sampling region; the algorithm was stopped early if no reduction to the cost of the solution was obtained in the last round of \textit{RRT*} compared to the previous round. Table~\ref{Tab:results} shows the benchmark results for a forest density of $3.2~\mathrm{trees}/m^2$.  The algorithms are compared in terms of their success rate, average mean path length, average maximum absolute velocity and average computing time. Although, the \textit{iSCP} and \textit{CHOMP} performed slightly better in path length and maximum velocity, respectively, the proposed approach excelled in success rate and its computing time was significantly lower. This success rate is inherited from the sampling-based
path planning algorithm (which is probabilistic complete),
since our method ensures the feasibility of the \textit{QP} problem.

Fig.~\ref{fig:Comp} shows a benchmark where the tree density was varied between $0.7$ and $3.2$ $\mathrm{trees}/m^2$. Typical trajectories are shown in Fig.~\ref{subfig:CompA}. For ease of visualization, the thickness of the paths and the radius of the trees are not drawn at scale. The comparison on computing time and success rate are shown in Fig.~\ref{subfig:CompB} and Fig.~\ref{subfig:CompC}, respectively.

\subsection{Experimental Results}
The test setup consisted of an Optitrack motion capture system with 15 infra-red cameras providing localization with sub-millimeter accuracy. The motion capture software and the trajectory planning algorithm were run on the same computer. The current position of the quadcopter and the corresponding reference point was broadcasted every 6 ms. The firmware of the \textit{Crazyflie 2.0} was modified to implement a custom nonlinear control algorithm for precise trajectory tracking.

\begin{figure}
\begin{center}
\includegraphics[width=8cm]{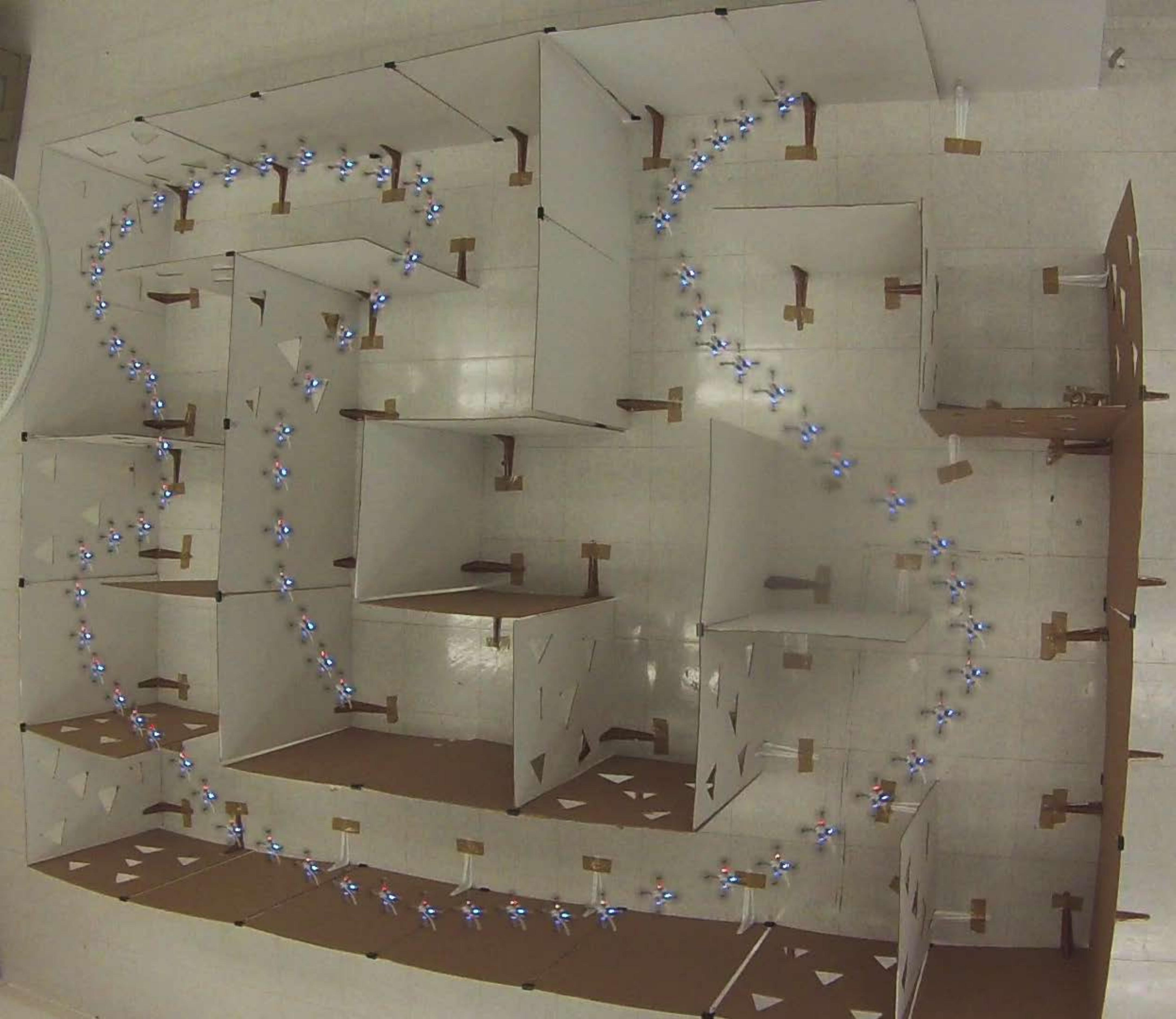}
\end{center}
\caption{A composite image of a quadcopter performing a trajectory in a 3x5 meters maze with average planning time of $\sim100$ms. }
\label{Fig:mazeExperiment}
\end{figure}

The first scenario, shown in Fig.~\ref{Fig:mazeExperiment}, is a $3m\times5m$  maze where the quadcopter had to navigate from an initial position inside the maze to a final goal. The task was repeated multiple times such that the goal of the last task became the initial position for the next. A new goal was generated by moving a marker inside the maze. The same maze environment was simulated in a computer, obtaining an average planning time of $\sim100$ms based on 500 trials with random start and end points.

The algorithm was also tested in a 3-dimensional scenario consisting of two hundred interwoven strings and twenty poles inside a volume of two cubic meters. This scenario, illustrated in Fig.~\ref{Fig:ExperimentStrings}, is similar to the one presented in~\cite{Landry2016} which consisted of 20 interwoven strings. Planning times in the order of minutes were reported in that work compared to a fraction of a second in the present work. The scenario was manually mapped using the motion capture system. During the experiment, new goals were provided online to the quadcopter while it navigated inside the complex environment. 

For these experiments, the radius $r$ of $\mathcal{B}(p)$ was selected to be $0.035m$ based on the quadcopter dimensions. Considering also a diameter of $0.005m$ for the strings and an average tracking error of $0.01m$, the maximum $\ell$ resulted in $0.035m$. The value of $\ell$ was set to $0.02$. The maximum acceleration in these experiments was set to $20m/s^2$ in each axis based on trial and error. Using these parameters, the algorithm took on average $\sim300$ms in a benchmark of 500 simulated trials. 

 \begin{table}[]
 \centering
 \vspace{1em}
 \begin{tabular}{c|cccc}
 \textbf{\begin{tabular}[c]{@{}c@{}}Parameter \\  $\mathbf{\ell}$ $\mathbf{[m]}$\\ \end{tabular}} & \textbf{\begin{tabular}[c]{@{}c@{}}Average\\ Mean Path\\  Length $\mathbf{[m]}$\end{tabular}} & \textbf{\begin{tabular}[c]{@{}c@{}}Average \\ Max\\ Velocity $\mathbf{[m/s]}$\end{tabular}} & \textbf{\begin{tabular}[c]{@{}c@{}}Average\\  Max\\ Acc $\mathbf{[m/s^2]}$\end{tabular}} & \textbf{\begin{tabular}[c]{@{}c@{}}Average \\ Compute\\ Time $\mathbf{[s]}$\end{tabular}} \\ \hline
0.010 & \textbf{3.6899} & 0.4391 &  \textbf{2.5828} & 0.9472   \\
0.015 & 3.8931 & 0.5386 & 2.6315 & 0.4750   \\
\textbf{0.020} & 3.9431 & 0.6226 & 2.6509 & 0.2922   \\
0.025 & 4.1916 & 0.6923 & 2.7354 & \textbf{0.2588}   \\
0.030 & 4.6593 & 0.7675 & 2.8683 & 0.2986   \\
0.035 & 4.7108 & \textbf{0.8268} & 2.8670 & 0.3853
 \end{tabular} 
 \caption{Comparison of the run time for different values of $\ell$.}
  \label{tab:TableComparison}
 \end{table}

Since the performance of the algorithm is affected by the desired minimum separation $\lmin$, and as a consequence by the parameter $\ell$, a benchmark varying the values of $\ell$ was performed for the same scenario. The results are presented in~Table~\ref{tab:TableComparison} where it can be seen that a low value of $\ell$ results in short trajectories, but higher computing time due to increased size of the \textit{QP} problem. On the other hand, a large value of $\ell$ results in a smaller \textit{QP} problem while a larger separation between the path and the trajectory is allowed and a higher-clearance is obtained, resulting in longer travels for the trajectories. Finally, it can be observed that the sampling-based planner takes longer to find a solution as $\ell$ is increased. This is expected since $\chi_{\mathrm{obs}}$ region size depends on $\ell$.

\section{Conclusions and Future Work}\label{sec:conclusions}
An algorithm for online trajectory planning and replanning of mobile robots was presented. The proposed method uses a sampling-based planning algorithm to efficiently determine an obstacle-free path. Then, an optimization program takes into account the robot dynamical constraints to generate a time dependent trajectory. The main contribution of this method lies in the formulation of the optimization problem that is guaranteed to be feasible, avoiding iterative formulations.
The effectiveness of this method was shown by applying it to the online planning of a quadcopter in multiple scenarios consisting of up to two hundred obstacles. The scenarios presented herein are some of the most obstacle-dense scenarios reported to date for online planning of a quadcopter.

Our approach was contrasted with state-of-the-art algorithms showing a significant improvement in computation time and success rate, which demonstrates its effectiveness for online planning in obstacle-dense environments.

As for future work, we aim to apply the algorithm in complex dynamic environments as well as its extension to multi-agent planning in obstacle-dense environments.

\appendix[Auxiliary results for the constraints derivation]
\label{Appendix}
Notice that, since the waypoints connecting $\eta[s]$ and $\eta[s+1]$ are placed in the line segment $\lambda(\eta[s+1]-\eta[s])+\eta[s]$, for some $\eta[s]$ and some $\lambda\in[0,1]$, it is always possible to find a coordinate transformation, for each pair $\eta[s]$, $\eta[s+1]$ with orthogonal axis, in such a way that  $\lambda(\eta[s+1]-\eta[s])+\eta[s]$ lie in the $x$-axis (with the $x$-axis pointing toward $\eta[s+1]$) of the new coordinate system. The notation $p_x$, (resp. $v_x$ and $a_x$) is used to represent the component of the position (resp. velocity and acceleration) along the $x$-axis.

Notice that, in this new coordinate system 
$
\Omega[k+1]=\left\lbrace \rho:\rho=\gamma+[\ell \ 0\ \cdots\ 0]^\trans, \gamma\in\Omega[k] \right\rbrace
$.
Additionally, without loss of generality, the origin is placed at $\varpi[k]$.
\begin{lemma}
Given the set of time-indexed waypoint regions~\eqref{Eq:WayPointRegion} and taking $A_{max}$ and $h$ as in~\eqref{def_h}, then for $p[k] \in \Omega[k]$ and $ v[k]$ such that $\| v[k]\|_\infty \leq \vmax$ and $v[k]^\trans(\varpi[k+1]-\varpi[k])\geq 0$, there exists a constant acceleration $ \|a[k]\|_\infty\leq\amax$ such that $ p[k+1] \in \Omega[k+1]$, $ \| v[k+1]\|_\infty\leq \vmax$ and $ v[k+1]^\trans(\varpi[k+1]-\varpi[k])\geq 0$.
\label{lemma:2l}
\end{lemma}
\begin{proof}
The proofs will be performed independently for each coordinate of the new coordinate system.

\textit{For the $x$--coordinate.} The conditions
$v[k]^\trans(\varpi[k+1]-\varpi[k])\geq 0$, $v[k+1]^\trans(\varpi[k+1]-\varpi[k])\geq 0$, $\|v[k]\|_\infty\leq\vmax$ and $\|v[k+1]\|_\infty\leq\vmax$ imply that $v_x[k]$ and $v_x[k+1]$ lie in $[0,\vmax]$. 

Let $p[k]\in \Omega[k+1]$ and $v_x[k]\in [0,\vmax]$ and notice that \eqref{def_h} implies $\amax=\dfrac{2\vmax}{h}$. Then, with the acceleration
$$
a_x[k] = \frac{\vmax - 2v_x[k]}{h}\in\left[-\dfrac{\vmax}{h},\dfrac{\vmax}{h}\right]\subset[-\amax,\amax]
$$
is obtained that
\begin{align*}
p_x[k+1]=p_x[k]+\ell\ \text{ and } \ v_x[k+1]=-v_x[k]+\vmax\in [0,\vmax]
\end{align*}

\textit{For the $y$--coordinate.}With $v_y[k]\in[-\vmax,\vmax]$, and

$$
a_y[k] = -\frac{2v_y[k]}{h} \in\left[-\frac{2\vmax}{h},\frac{2\vmax}{h}\right]=[-\amax,\amax]
$$
Then $p_y[k+1]=p_y[k]$ and $v_y[k+1]=-v_y[k]$.

\textit{For the rest of the coordinates.} The argument is equivalent as in the $y$--coordinate.
The above results shows that
$$
p[k+1] = p[k] + [\ell\ 0\ \cdots \ ]^\trans \in \Omega[k+1].
$$

Additionally, $v[k+1]]^\trans(\eta[s+1]-\eta[s])\geq 0$ in the new coordinate system is $[v_x[k+1]\ v_y[k+1]\ \cdots][\alpha\ 0\ \cdots\ 0]^\trans \geq 0$, for some $\alpha>0$. In this way $[v_x[k+1]\ v_y[k+1]\ \cdots][\alpha\ 0\ \cdots\ 0]^\trans = \alpha v_x[k+1] \geq 0$, which completes the proof.
\end{proof}

\begin{lemma}
Let $\varpi[k]=\varpi[k+1]=\eta[s]$ and let $A_{max}$ and $h$ be given as in~\eqref{def_h}. Then given the set of time indexed waypoint regions~\eqref{Eq:WayPointRegion} for $p[k] \in \Omega[k] $, and $v[k]$ satisfying $ \|v[k]\|_\infty \leq \vmax$, there exists a pair of constant accelerations  $\|a[k]\|_\infty\leq\amax$ and  $\|a[k+1]\|_\infty\leq\amax$,  such that $ p[k+1] \in \Omega[k+1] $, $ \|v[k+1]\|_\infty \leq\vmax$, $ p[k+2] \in \Omega[k+2]$ and $ \|v[k+2]\|_\infty\leq\vmax$ with $v[k+2]^\trans(\varpi[k+2]-\varpi[k+1])\geq 0$.
\label{lemma:Vmax}
\end{lemma}
\begin{proof}
The proofs will be performed independently for each coordinate of the new coordinate system.

\textit{For the $x$--coordinate.}
The transformation changes the constraints as follows. For $\varpi[k]$ placed at the origin, $p_x[k]\in[-\alpha \ell,\alpha \ell]$, $v_x[k]\in[-\alpha \vmax,\alpha \vmax]$, 
$a_x[k]\in[-\alpha \amax,\alpha \amax]$ for some $\alpha\in[1,\sqrt{d}]$. Thus, in the new coordinate system $p_x[k]$ can be written as $p_x[k] = \alpha\ell(2\lx - 1)$, $\lx \in[0,1]$ and $v_x[k] = \alpha\vmax(2\lv - 1)$,  $\lv \in[0,1]$ and choosing
\begin{align*}
a_1 &= \alpha\amax\left(1-\frac{\lx}{2}-\frac{3\lv}{2}\right)\in[-\alpha\amax,\alpha\amax]\\
a_2 &= \alpha\amax\left(\lxxx +\lxx-\frac{1}{2}\right)\in\left[-\alpha\amax,\alpha\amax\right]
\end{align*}
with
\begin{align*}
\lxx &= \frac{1}{2}(\lx+\lv)\in[0,1], \ \ \lvv = 1-\lxx\in[0,1] \\
\lxxx &= \frac{V_f}{2\vmax}\in\left[0,\dfrac{1}{2}\right], \ \ V_f \in\left[0,\dfrac{\vmax}{\alpha}\right]
\end{align*}
is obtained that
\begin{align*}
p_x[k+1] &= \alpha\ell(2\lxx-1)\in[-\alpha\ell,\alpha\ell] \\
v_x[k+1] &= \alpha\vmax(2\lvv -1)\in[-\alpha\vmax,\alpha\vmax]\\
p_x[k+2] &= 2\alpha\ell\lxxx\in\left[0,\ell\right], \ \ \ 
v_x[k+2] = \alpha V_f
\end{align*}

\textit{For the rest of the coordinates.}
The result for the $y$-coordinate in the proof of Lemma~\ref{lemma:2l} can be used for the interval $[kh,(k+1)h]$ and for $[(k+1)h,(k+2)h]$, so that $p_y[k]=p_y[k+1]$  with feasible velocities and accelerations and the same for the rest of the coordinates.
The above results can be used to prove that $p[k+1]\in\Omega[k+1]$ and $p[k+2]\in\Omega[k+2]$ while $v[k+2]^\trans(\eta[s+1]-\eta[s])\geq0$ in a similar way as in the proof of Lemma~\ref{lemma:2l}, which completes the proof.
\end{proof}

\begin{lemma}
Given the set of time indexed waypoint regions~\eqref{Eq:WayPointRegion} and taking $A_{max}$ and $h$ as in~\eqref{def_h}, then for $p[k]\in\Omega[k]$ and $\|v[k]\|_\infty\leq\vmax$ and with $\|a[k]\|_\infty\leq \amax$ such that $p[k+1]\in\Omega[k+1]$ and $\|v[k+1]\|_\infty\leq\vmax$, then $|b(t)|\leq\frac{3}{2}\ell\sqrt{d}$ for $t\in(kh,(k+1)h)$.
\label{lemma:bbar}
\end{lemma}
\begin{proof}
To show this argument is not required to rotate the coordinate system but only to place the origin at $\varpi[k]$. Taking into account first only the $x$-coordinate, the biggest separation for $p_x(t)$, $t\in[kh,(k+1)h]$ occurs when $p_x[k]=\pm\ell$ and $v_x[k]=\pm\vmax$. The only way that $p_x[k+1] \in [-\ell,\ell]$ and $v_x[k+1]\in[-\vmax,\vmax]$ for feasible $a_x[k]$ is that $p_x[k+1]=p_x[k]$ and $v[k+1]=-v[k]$. For this condition it is required that $a_x[k] = \mp\amax$. Due to the symmetry of the problem, the maximum value of $p_x(t)$ for $t\in(kh,(k+1)h)$ occurs when $t=\tcritic = \left(k+\frac{1}{2}\right)h$ so that $x(\tcritic) = \frac{3}{2}\ell$.
If every coordinate has a maximum separation of $\frac{3}{2}\ell$ then $b(t)$ will have a bound given by $|b(t)|\leq \frac{3}{2}\ell\sqrt{d}$.
\end{proof}

\section*{Acknowledgements}
The authors would like to thank Kirk Skeba and Maynard Falconer for the fruitful discussions on this research topic and the anonymous reviewers for their constructive comments. 


\begin{thebibliography}{10}
\providecommand{\url}[1]{#1}
\csname url@rmstyle\endcsname
\providecommand{\newblock}{\relax}
\providecommand{\bibinfo}[2]{#2}
\providecommand\BIBentrySTDinterwordspacing{\spaceskip=0pt\relax}
\providecommand\BIBentryALTinterwordstretchfactor{4}
\providecommand\BIBentryALTinterwordspacing{\spaceskip=\fontdimen2\font plus
\BIBentryALTinterwordstretchfactor\fontdimen3\font minus
  \fontdimen4\font\relax}
\providecommand\BIBforeignlanguage[2]{{%
\expandafter\ifx\csname l@#1\endcsname\relax
\typeout{** WARNING: IEEEtran.bst: No hyphenation pattern has been}%
\typeout{** loaded for the language `#1'. Using the pattern for}%
\typeout{** the default language instead.}%
\else
\language=\csname l@#1\endcsname
\fi
#2}}

\bibitem{Katrakazas2015}
C.~Katrakazas, M.~Quddus, W.-H. Chen, and L.~Deka, ``Real-time motion planning
  methods for autonomous on-road driving: State-of-the-art and future research
  directions,'' \emph{Transportation Research Part C: Emerging Technologies},
  vol.~60, pp. 416 -- 442, November 2015.

\bibitem{Paden2016}
B.~Paden, M.~Cap, S.~Z. Yong, D.~Yershov, and E.~Frazzoli, ``A survey of motion
  planning and control techniques for self-driving urban vehicles,'' \emph{IEEE
  Trans. Intell. Veh.}, vol.~1, no.~1, pp. 33--55, March 2016.

\bibitem{Karaman2011}
S.~Karaman and E.~Frazzoli, ``Sampling-based algorithms for optimal motion
  planning,'' \emph{Int. J. Robot. Res.}, vol.~30, no.~7, pp. 846--894, June
  2011.

\bibitem{Elbanhawi2014}
M.~Elbanhawi and M.~Simic, ``Sampling-based robot motion planning: A review,''
  \emph{IEEE Access}, vol.~2, pp. 56--77, February 2014.

\bibitem{Gammell2015}
J.~D. Gammell, S.~S. Srinivasa, and T.~D. Barfoot, ``Batch informed trees
  ({BIT*}): Sampling-based optimal planning via the heuristically guided search
  of implicit random geometric graphs,'' in \emph{IEEE Int. Conf. Robot.
  Autom.}, June 2015, pp. 3067--3074.

\bibitem{Choudhury2016}
S.~Choudhury, J.~D. Gammell, T.~D. Barfoot, S.~S. Srinivasa, and S.~Scherer,
  ``Regionally accelerated batch informed trees ({RABIT}*): A framework to
  integrate local information into optimal path planning,'' in \emph{IEEE Int.
  Conf. Robot. Autom.}\hskip 1em plus 0.5em minus 0.4em\relax IEEE, May 2016,
  pp. 4207--4214.

\bibitem{Deits2015}
R.~Deits and R.~Tedrake, ``Computing large convex regions of obstacle-free
  space through semidefinite programming,'' in \emph{Algorithmic Foundations of
  Robotics XI}.\hskip 1em plus 0.5em minus 0.4em\relax Springer, August 2015,
  pp. 109--124.

\bibitem{Landry2016}
B.~Landry, R.~Deits, P.~R. Florence, and R.~Tedrake, ``Aggressive quadrotor
  flight through cluttered environments using mixed integer programming,'' in
  \emph{IEEE Int. Conf. Robot. Autom.}, May 2016, pp. 1469--1475.

\bibitem{Webb2013}
D.~J. Webb and J.~van~den Berg, ``Kinodynamic {RRT}*: Asymptotically optimal
  motion planning for robots with linear dynamics,'' in \emph{IEEE Int. Conf.
  Robot. Autom.}\hskip 1em plus 0.5em minus 0.4em\relax IEEE, May 2013, pp.
  5054--5061.

\bibitem{Xie2015}
C.~Xie, J.~van~den Berg, S.~Patil, and P.~Abbeel, ``Toward asymptotically
  optimal motion planning for kinodynamic systems using a two-point boundary
  value problem solver,'' in \emph{IEEE Int. Conf. Robot. Autom.}\hskip 1em
  plus 0.5em minus 0.4em\relax IEEE, May 2015, pp. 4187--4194.

\bibitem{Richter2016}
C.~Richter, A.~Bry, and N.~Roy, ``Polynomial trajectory planning for aggressive
  quadrotor flight in dense indoor environments,'' in \emph{Robotics
  Research}.\hskip 1em plus 0.5em minus 0.4em\relax Springer, April 2016, pp.
  649--666.

\bibitem{Augugliaro2012}
F.~Augugliaro, A.~P. Schoellig, and R.~D'Andrea, ``Generation of collision-free
  trajectories for a quadrocopter fleet: A sequential convex programming
  approach,'' in \emph{IEEE/RSJ Int. Conf. Intell. Robots Syst.}, October 2012,
  pp. 1917--1922.

\bibitem{Chen2015decoupled}
Y.~Chen, M.~Cutler, and J.~P. How, ``Decoupled multiagent path planning via
  incremental sequential convex programming,'' in \emph{IEEE Int. Conf. Robot.
  Autom.}, May 2015, pp. 5954--5961.

\bibitem{Schulman2014}
J.~Schulman, Y.~Duan, J.~Ho, A.~Lee, I.~Awwal, H.~Bradlow, J.~Pan, S.~Patil,
  K.~Goldberg, and P.~Abbeel, ``Motion planning with sequential convex
  optimization and convex collision checking,'' \emph{Int. J. Robot. Res.},
  vol.~33, no.~9, pp. 1251--1270, June 2014.

\bibitem{Chen2016}
J.~Chen, T.~Liu, and S.~Shen, ``Online generation of collision-free
  trajectories for quadrotor flight in unknown cluttered environments,'' in
  \emph{IEEE Int. Conf. Robot. Autom.}, May 2016, pp. 1476--1483.

\bibitem{Zucker2013}
M.~Zucker, N.~Ratliff, A.~D. Dragan, M.~Pivtoraiko, M.~Klingensmith, C.~M.
  Dellin, J.~A. Bagnell, and S.~S. Srinivasa, ``{CHOMP}: Covariant hamiltonian
  optimization for motion planning,'' \emph{Int. J. Robot. Res.}, vol.~32, no.
  9-10, pp. 1164--1193, September 2013.

\bibitem{Kalakrishnan2011}
M.~Kalakrishnan, S.~Chitta, E.~Theodorou, P.~Pastor, and S.~Schaal, ``{STOMP}:
  Stochastic trajectory optimization for motion planning,'' in \emph{IEEE Int.
  Conf. Robot. Autom.}\hskip 1em plus 0.5em minus 0.4em\relax IEEE, May 2011,
  pp. 4569--4574.

\bibitem{Oleynikova2016}
H.~Oleynikova, M.~Burri, Z.~Taylor, J.~Nieto, R.~Siegwart, and E.~Galceran,
  ``Continuous-time trajectory optimization for online {UAV} replanning,'' in
  \emph{IEEE/RSJ Int. Conf. Intell. Robots Syst.}, October 2016, pp.
  5332--5339.

\bibitem{Loianno2017}
G.~Loianno, C.~Brunner, G.~McGrath, and V.~Kumar, ``Estimation, control, and
  planning for aggressive flight with a small quadrotor with a single camera
  and {IMU},'' \emph{IEEE Robot Autom. Lett.}, vol.~2, no.~2, pp. 404--411,
  November 2017.

\bibitem{Jaillet2010}
L.~Jaillet, J.~Cortés, and T.~Siméon, ``Sampling-based path planning on
  configuration-space costmaps,'' \emph{IEEE Trans. Robot}, vol.~26, no.~4, pp.
  635--646, August 2010.

\bibitem{Devaurs2016}
D.~Devaurs, T.~Siméon, and J.~Cortés, ``Optimal path planning in complex cost
  spaces with sampling-based algorithms,'' \emph{IEEE Trans. Autom. Sci. Eng.},
  vol.~13, no.~2, pp. 415--424, April 2016.

\bibitem{Gammell2014}
J.~D. Gammell, S.~S. Srinivasa, and T.~D. Barfoot, ``Informed {RRT}*: Optimal
  sampling-based path planning focused via direct sampling of an admissible
  ellipsoidal heuristic,'' in \emph{IEEE/RSJ Int. Conf. Intell. Robots Syst.},
  September 2014, pp. 2997--3004.

\bibitem{Karaman2011Anytime}
S.~Karaman, M.~R. Walter, A.~Perez, E.~Frazzoli, and S.~Teller, ``Anytime
  motion planning using the {RRT}*,'' in \emph{IEEE Int. Conf. Robot.
  Autom.}\hskip 1em plus 0.5em minus 0.4em\relax IEEE, May 2011, pp.
  1478--1483.

\bibitem{Neunert2016}
M.~Neunert, C.~de~Crousaz, F.~Furrer, M.~Kamel, F.~Farshidian, R.~Siegwart, and
  J.~Buchli, ``Fast nonlinear model predictive control for unified trajectory
  optimization and tracking,'' in \emph{IEEE Int. Conf. Robot. Autom.}, May
  2016, pp. 1398--1404.

\bibitem{Mosek}
\BIBentryALTinterwordspacing
{Mosek ApS}, \emph{The MOSEK optimization software}, 2016. [Online]. Available:
  \url{http://www.mosek.com/}
\BIBentrySTDinterwordspacing

\bibitem{KaramanForest}
S.~Karaman and E.~Frazzoli, ``High-speed flight in an ergodic forest,'' in
  \emph{IEEE Int. Conf. Robot. Autom.}, May 2012, pp. 2899--2906.

\bibitem{CrazyFlie}
\BIBentryALTinterwordspacing
{Bitc raze AB}, \emph{A quadcopter open platform}, 2016. [Online]. Available:
  \url{https://www.bitcraze.io/crazyflie-2/}
\BIBentrySTDinterwordspacing

\end{thebibliography}

\end{document}